\documentclass{article}

\usepackage{microtype}
\usepackage{graphicx}
\usepackage{hyperref}       
\usepackage{booktabs} 
\usepackage{multicol}
\usepackage{amssymb}
\usepackage[linesnumbered,ruled,noend]{algorithm2e}
\usepackage{amsmath}
\usepackage{amsthm}
\usepackage{algpseudocode}
\usepackage[dvipsnames]{xcolor}
\usepackage{colortbl}
\usepackage{natbib}
\usepackage{subcaption}
\usepackage{wrapfig}
\usepackage{xspace} 
\usepackage{enumitem}
\usepackage{tikz}
\usetikzlibrary{tikzmark,calc}
\usepackage{etoc}
\etocdepthtag.toc{mtchapter}
\etocsettagdepth{mtchapter}{subsection}
\etocsettagdepth{mtappendix}{none}

\newcommand{\ditto}{\texttt{Ditto}\xspace}

\newcommand{\var}{\textup{var}}

\newtheorem{definition}{Definition}

\newtheorem{theorem}{Theorem}
\newtheorem*{theorem*}{Theorem}
\newtheorem{remark}{Remark}
\newtheorem{lemma}{Lemma}
\newtheorem*{lemma*}{Lemma}
\newtheorem{corollary}{Corollary}
\newtheorem*{corollary*}{Corollary}

\DeclareMathOperator*{\argmin}{arg\,min}

\definecolor{myred}{HTML}{fff5e6}

\newcommand{\mytitle}{\ditto: Fair and Robust Federated Learning Through Personalization}

\usepackage[accepted]{icml2021}

\icmltitlerunning{\mytitle}

\begin{document}

\twocolumn[
\icmltitle{\mytitle}


\begin{icmlauthorlist}
\icmlauthor{Tian Li}{cmu}
\icmlauthor{Shengyuan Hu}{cmu}
\icmlauthor{Ahmad Beirami}{fb}
\icmlauthor{Virginia Smith}{cmu}
\end{icmlauthorlist}

\icmlaffiliation{cmu}{Carnegie Mellon University}
\icmlaffiliation{fb}{Facebook AI}

\icmlcorrespondingauthor{Tian Li}{tianli@cmu.edu}


\vskip 0.3in
]



\printAffiliationsAndNotice{}  

\begin{abstract}

Fairness and robustness are two important concerns for federated learning systems. In this work, we identify that \emph{robustness} to data and model poisoning attacks and \emph{fairness}, measured as the uniformity of performance across devices, are competing constraints in statistically heterogeneous networks. 
To address these constraints, we propose employing a simple, general framework for personalized federated learning, \ditto, that can inherently provide fairness and robustness benefits, and develop a scalable solver for it. 
Theoretically, we  analyze the ability of \ditto to achieve
fairness and robustness simultaneously on a class of linear problems.
Empirically, across a suite of federated datasets, we show that \ditto not only achieves  competitive performance relative to recent personalization methods, but also enables more accurate, robust, and fair models relative to state-of-the-art fair or robust baselines.
\end{abstract}

\section{Introduction}\label{sec:introduction}

Federated learning (FL) aims to collaboratively learn from data that has been generated by, and resides on, a number of remote devices or servers~\cite{mcmahan2017communication}. FL stands to produce highly accurate statistical models by aggregating knowledge from disparate data sources. 
However, to deploy FL in practice, it is necessary for the resulting systems to be not only accurate, but to also satisfy a number of pragmatic constraints regarding issues such as fairness, robustness, and privacy. Simultaneously satisfying these varied constraints can be exceptionally difficult~\cite{kairouz2019advances}.

We focus in this work specifically on issues of accuracy, fairness (i.e., limiting performance disparities across the network~\cite{mohri2019agnostic}), and robustness (against training-time data and model poisoning attacks).
Many prior efforts have separately considered fairness or robustness in federated learning. For instance,  fairness strategies include using minimax optimization to focus on the worst-performing devices~\cite{mohri2019agnostic,hu2020fedmgda+} or reweighting the devices to allow for a flexible fairness/accuracy tradeoff~\cite{li2019fair,li2020tilted}. Robust methods commonly use techniques such as gradient clipping~\cite{Sun2019CanYR} or robust aggregation~\cite{Blanchard2017MachineLW, pmlr-v80-yin18a}. 

While these approaches may be effective at either promoting fairness or defending against training-time attacks in isolation, we show that the constraints of fairness and robustness can directly compete with one another when training a single global model, and that simultaneously optimizing for accuracy, fairness, and robustness requires careful consideration. For example, as we empirically demonstrate (Section~\ref{sec:exp}), current fairness approaches can render FL systems highly susceptible to training time attacks from malicious devices. On the other hand, robust methods may filter out rare but informative updates, causing unfairness~\cite{wang2020attack}.

In this work, we investigate a simple, scalable technique to simultaneously improve accuracy, fairness, and robustness in federated learning. While addressing the competing constraints of FL may seem like an insurmountable problem,  we identify that {statistical heterogeneity} (i.e., non-identically distributed data) is a root cause for tension between these constraints, and is key in paving a path forward. In particular, we suggest that methods for personalized FL---which model and adapt to the heterogeneity in federated settings by learning distinct models for each device---may provide \textit{inherent} benefits in terms of fairness and robustness.

To explore this idea, we propose \ditto, a scalable federated multi-task learning framework. 
\ditto can be seen as a lightweight personalization add-on for standard global FL. It is applicable to both convex and non-convex objectives, and inherits similar privacy and efficiency properties as traditional FL. We evaluate \ditto on a suite of federated benchmarks and show that, surprisingly, this simple form of personalization can in fact deliver better accuracy, robustness, and fairness benefits than state-of-the-art, problem-specific objectives that consider these constraints separately. We summarize our contributions below:

\vspace{-.1in}

\begin{itemize}[leftmargin=*]
\setlength\itemsep{0em}
    \item We propose \ditto, a multi-task learning objective for federated learning that provides personalization while retaining similar efficiency and privacy benefits as traditional FL. We provide convergence guarantees for our proposed \ditto solver, which incorporate common practices in cross-device federated learning such as limited device participation and local updating. Despite its simplicity, we show that \ditto can deliver similar or superior accuracy relative to other common methods for personalized federated learning.
    \item Next, we demonstrate that the benefits of \ditto go beyond accuracy---showing that the personalized objective can inherently offer \textit{robustness} superior to that of common robust FL methods across a diverse set of data and model poisoning attacks. On average across all datasets and attacks, \ditto improves test accuracy by $\sim$6\% (absolute) over the strongest robust baseline.  
    \item Similarly, we show that \ditto can naturally increase \textit{fairness}---reducing variance of the test accuracy across devices by $\sim$10\% while  maintaining similar or superior accuracy relative to state-of-the-art methods for fair FL.
    \item Finally, we highlight that \ditto  is particularly useful for practical applications where we {simultaneously} care about multiple constraints (accuracy, fairness, and robustness). We motivate this through analysis on a toy example in Section~\ref{sec:ditto}, as well as experiments across a suite of federated datasets in Section~\ref{sec:exp}. 
\end{itemize}

\section{Background \& Related Work} \label{sec:related_wrok}

Robustness and fairness are two broad areas of research that extend well beyond the application of federated learning. In this section we provide precise definitions of the notions of robustness/fairness considered in this work, and give an overview of prior work in robustness, fairness, and personalization in the context of federated learning.

\paragraph{Robustness in Federated Learning.}
Training-time attacks (including data poisoning and model poisoning) have been extensively studied in prior work~\cite{Biggio2012PoisoningAA,Gu2017BadNetsIV,chen2017targeted,shafahi2018poison,Liu2018TrojaningAO,huang2020metapoison,Xie2020DBA,wang2020attack,dumford2018backdooring,huang2020metapoison}. 
In federated settings, a number of strong attack methods have been explored, including scaling malicious model updates~\cite{bagdasaryan2020backdoor}, 
collaborative attacking~\cite{sun2020data}, defense-aware attacks~\cite{bhagoji2019analyzing,fang2020local}, and adding edge-case adversarial training samples~\cite{wang2020attack}.
Our work aims to investigate common attacks related to Byzantine robustness~\cite{lamport2019byzantine}, as formally described  below.

\vspace{.2em}
\begin{definition}[Robustness] \label{def:robustness}
We are conceptually interested in Byzantine robustness~\cite{lamport2019byzantine}, where the malicious devices can send arbitrary updates to the server to compromise training. To measure robustness, we assess the mean test performance on benign devices, i.e., we consider model $w_1$ to be more robust than $w_2$ to a specific attack if the mean test performance across the benign devices is higher for model $w_1$ than $w_2$ after training with the attack. We examine three widely-used attacks in our threat model:
\vspace{-0.5em}
\end{definition}
\begin{itemize}[leftmargin=*]
     \setlength\itemsep{.2em}
    \item \textit{(A1) Label poisoning}: Corrupted devices do not have access to the training APIs and training samples are poisoned with flipped (if binary) or uniformly random noisy labels~\cite{bhagoji2019analyzing,biggio2011support}. 
    \item \textit{(A2) Random updates}: Malicious devices send random zero-mean Gaussian parameters~\cite{xu2020towards}. \item \textit{(A3) Model replacement}: Malicious devices scale their adversarial updates to make them dominate the aggregate updates~\cite{bagdasaryan2020backdoor}.
\end{itemize}

\vspace{-5pt}
While non-exhaustive, these attacks have been commonly studied in distributed and federated settings, and explore corruption at various points (the underlying data, labels, or model). 
In terms of defenses, robust aggregation is a common strategy to mitigate the effect of malicious updates~\cite{Blanchard2017MachineLW,pillutla2019robust,Sun2019CanYR,li2019rsa,he2020byzantine}. 
Other defenses include 
gradient clipping~\cite{Sun2019CanYR} or normalization~\cite{hu2020fedmgda+}.
While these strategies can improve robustness, 
they may also produce \textit{unfair} models by filtering out informative updates, especially in heterogeneous settings~\cite{wang2020attack}. In our experiments (Section~\ref{sec:exp}), we compare \ditto with several strong defenses (median, gradient clipping~\cite{Sun2019CanYR}, Krum, Multi-Krum~\cite{Blanchard2017MachineLW}, gradient-norm based anomaly detector~\cite{bagdasaryan2020backdoor}, and a new defense proposed herein) and show that \ditto can improve both robustness and fairness compared with these methods.

\paragraph{Fairness in Federated Learning.} Due to the heterogeneity of the data in federated networks, it is possible that the performance of a model will vary significantly across the devices. This concern, also known as \textit{representation disparity}~\cite{hashimoto2018fairness}, is a major challenge in FL, as it can potentially result in uneven outcomes for the devices. Following~\citet{li2019fair}, we provide a more formal definition of this fairness in the context of FL below:

\begin{definition}[Fairness] \label{def:fairness}
We say that a model $w_1$ is more \textit{fair} than $w_2$ if the test performance distribution of $w_1$ across the network is \textup{more uniform} than that of $w_2$, i.e., $\textup{std} \left\{F_k(w_1) \right\}_{k \in [K]} < \textup{std} \left\{F_k(w_2)\right\}_{k \in [K]}$ where $F_k(\cdot)$ denotes the test loss on device $k$$\in$$[K]$, and $\textup{std}\{\cdot\}$ denotes the standard deviation. In the presence of adversaries, we measure fairness only on benign devices. 
\end{definition}
\vspace{-0.025in}
We note that there exists a tension between variance and utility in the definition above; in general, a common goal is to lower the variance while maintaining a reasonable average performance (e.g., average test accuracy). To address representation disparity, it is common to use minimax optimization~\cite{mohri2019agnostic,deng2020distributionally} or flexible sample reweighting approaches~\cite{li2019fair, li2020tilted} to encourage a more uniform quality of service.
In all cases, by up-weighting the importance of rare devices or data, fair methods may not be robust in that they can easily overfit to corrupted devices (see Section~\ref{sec:exp:tension}). The tension between fairness and robustness has been studied in previous works, though for different notions of fairness (equalized odds) or robustness (backdoor attacks)~\cite{wang2020attack}, or in centralized settings~\cite{chang2020adversarial}.  Recently,~\citet{hu2020fedmgda+} proposed FedMGDA+, a method targeting fair and robust FL; however, this work combines classical fairness (minimax optimization) and robustness (gradient normalization) techniques, in contrast to the multi-task framework proposed herein, which we show can \textit{inherently} provide benefits with respect to both constraints simultaneously.

\vspace{-0.05in}
\paragraph{Personalized Federated Learning.}
Given the variability of data in federated networks, personalization is a natural approach used to improve accuracy. Numerous works have proposed techniques for personalized federated learning. \citet{smith2017federated} first explore personalized FL via a primal-dual MTL framework, which applies to convex settings. Personalized FL has also been explored through clustering~\citep[e.g.,][]{ghosh2020efficient,sattler2020clustered,muhammad2020fedfast}, finetuning/transfer learning~\cite{zhao2018federated,yu2020salvaging},  meta-learning~\cite{jiang2019improving,chen2018federated,khodak2019adaptive,fallah2020personalized,li2019differentially,singhal2021federated}, and other forms of MTL, such as hard model parameter sharing~\cite{agarwal2020federated,liang2020think} or the weighted combination method in~\citet{zhang2020personalized}. Our work differs from these approaches by simultaneously learning local and global models via a global-regularized MTL framework, which applies to non-convex ML objectives.  

Similar in spirit to our approach are works that interpolate between global and local models~\cite{mansour2020three,deng2020adaptive}. However, as discussed in~\citet{deng2020adaptive}, these approaches can effectively reduce to local minimizers without additional constraints. 
The most closely related works are those that regularize personalized models towards their average~\cite{hanzely2020federated,hanzely2020lower,dinh2020personalized}, which can be seen as a form of classical mean-regularized MTL~\cite{evgeniou2004regularized}. Our objective is similarly inspired by mean-regularized MTL, although we regularize towards a global model rather than the average personalized model. As we discuss in Section~\ref{sec:ditto}, one advantage of this is that it allows for methods designed for the global federated learning problem (e.g., optimization methods, privacy/security mechanisms) to be easily re-used in our framework, with the benefit of additional personalization. 
We compare against a range of personalized methods empirically in Section~\ref{sec:exp:other_properties}, showing that \ditto achieves similar or superior performance across a number of common FL benchmarks.

Finally, a key contribution of our work is jointly exploring the robustness and fairness benefits of personalized FL. The benefits of personalization for fairness alone have been demonstrated empirically in prior work~\citep{wang2019federated,hao2020waffle}. Connections between personalization and robustness have also been explored in~\citet{yu2020salvaging}, although the authors propose using personalization methods on top of robust mechanisms. Our work differs from these works by arguing that MTL itself offers inherent robustness and fairness benefits, and exploring the challenges that exist when attempting to satisfy both constraints simultaneously. %

\section{\ditto: Global-Regularized Federated Multi-Task Learning} \label{sec:ditto}

In order to explore the possible fairness/robustness benefits of personalized FL, we first propose a simple and scalable framework for federated multi-task learning. As we will see, this lightweight personalization framework is amenable to analyses while also having strong empirical performance. 
 We explain our proposed objective, \ditto, in Section~\ref{sec:obj} and then present a scalable algorithm to solve it in federated settings (Section~\ref{sec:solver}). We provide convergence guarantees for our solver, and explain several practical benefits of our modular approach in terms of privacy and efficiency.
 Finally, in  Section~\ref{sec:theory}, we characterize the benefits of \ditto in terms of fairness and robustness on a class of linear problems. We empirically explore the fairness and robustness properties against state-of-the-art baselines in Section~\ref{sec:exp}.

\subsection{\ditto Objective} \label{sec:obj}

Traditionally, federated learning objectives consider fitting a single global model, $w$, across all local data in the network. The aim is to solve: 
\begin{equation}\label{obj:global}
    \min_w \, G(F_1(w), \dots\, F_K(w)) \, ,  \tag{Global Obj}
\end{equation}
where $F_k(w)$ is the local objective for device $k$, and  $G(\cdot)$ is a function that aggregates the local objectives $\{F_k(w)\}_{k \in [K]}$ from each device. For example, in FedAvg~\cite{mcmahan2017communication}, $G(\cdot)$ is typically set to be a weighted average of local losses, i.e., $\sum_{k=1}^K p_k F_k(w)$, where $p_k$ is a pre-defined non-negative weight such that $\sum_k p_k=1$. 

However, in general, each device may generate data $x_k$ via a distinct distribution $\mathcal{D}_k$, i.e., $F_k(w) := \mathbb{E}_{x_k \sim \mathcal{D}_k}\left[f_k(w; x_k)\right]$. 
To better account for this heterogeneity, it is common to consider techniques that learn personalized, device-specific models, $\{v_k\}_{k \in [K]}$ across the network. 
In this work we explore personalization through a simple framework for federated multi-task learning. We consider two `tasks': the global objective~\eqref{obj:global}, and the local objective $F_k(v_k)$, which aims to learn a model using only the data of device $k$. To relate these tasks, we incorporate a regularization term that encourages the personalized models to be close to the optimal global model.
The resulting bi-level optimization problem for each device $k \in [K]$ is given by:
\vspace{-0.1in}
\begin{equation*}\label{obj:multitask}
\begin{aligned} 
 \min_{v_k }\quad & h_k(v_k; w^*) := F_k(v_k) + \frac{\lambda}{2}\left\|v_k-w^*\right\|^2  
 \\
  \text{s.t.} \quad & w^* \in \argmin_w G(F_1(w), \dots\, F_K(w)))\nonumber \,. 
\end{aligned}
\tag{\ditto}
\end{equation*}

Here the hyperparameter $\lambda$ controls the interpolation between local and global models. When $\lambda$ is set to 0, \ref{obj:multitask}
 is reduced to training local models; as $\lambda$ grows large, it recovers  global model objective (\ref{obj:global}) ($\lambda \to + \infty$). 
 
 \paragraph{Intuition for Fairness/Robustness Benefits.} In addition to improving accuracy via personalization, we argue that \ditto can offer  fairness and robustness benefits. To reason about this, consider a simple case where data are \textit{homogeneous} across devices. Without adversaries, learning a single global model is optimal for generalization. However, in the presence of adversaries, learning globally might introduce  corruption, while learning local models may not generalize well due to limited sample size. \ditto with an appropriate value of $\lambda$ offers a tradeoff between these two extremes: the smaller $\lambda$, the more the personalized models $v_k$ can deviate from the (corrupted) global model $w$, potentially providing robustness at the expense of generalization. 
 In the heterogeneous case (which can lead to issues of unfairness as described in Section~\ref{sec:related_wrok}), a finite $\lambda$ exists to offer robustness and fairness jointly. We explore these ideas more rigorously in Section~\ref{sec:theory} by analyzing the tradeoffs between accuracy, fairness, and robustness in terms of $\lambda$ for a class of linear regression problems, and demonstrate fairness/robustness benefits of \ditto empirically in Section~\ref{sec:exp}.

\paragraph{Other Personalization Schemes.}
As discussed in Section~\ref{sec:related_wrok}, personalization is a widely-studied topic in FL. Our intuition in \ditto is that personalization, by reducing reliance on the global model, can  reduce representation disparity (i.e., unfairness) and potentially improve robustness. It is possible that other personalization techniques beyond \ditto offer similar benefits: We provide some initial, encouraging results on this in Section~\ref{sec:exp:other_properties}.
However, we specifically explore \ditto due to its simple nature, scalability, and strong empirical performance.   \ditto is closely related to works that regularize personalized models towards their average~\cite{hanzely2020federated,hanzely2020lower,dinh2020personalized}, similar to classical mean-regularized MTL~\cite{evgeniou2004regularized}; \ditto differs by regularizing towards a global model rather than the average personalized model. We find that this provides benefits in terms of \textit{analysis} (Section~\ref{sec:theory}), as we can easily reason about \ditto relative to the global ($\lambda \to \infty$) vs. local ($\lambda \to 0$) baselines; \textit{empirically}, in terms of accuracy, fairness, and robustness (Section~\ref{sec:exp}); and \textit{practically}, in terms of the modularity it affords our corresponding solver (Section~\ref{sec:solver}).

\paragraph{Other Regularizers.} To encourage the personalized models $v_k$ to be close to the optimal global model $w^*$, there are choices beyond the $L_2$ norm that could be considered, e.g., using a Bregman divergence-based regularizer or reshaping the $L_2$ ball using the Fisher information matrix. Under the logistic loss (used in our experiments), the Bregman divergence will reduce to KL divergence, and its second-order Taylor expansion will result in an $L_2$ ball reshaped with the Fisher information matrix. Such regularizers are studied in other related contexts like continual learning~\cite{kirkpatrick2017overcoming,schwarz2018progress}, multi-task learning~\cite{yu2020salvaging}, or finetuning for language models~\cite{jiang2019smart}. 
However, in our experiments (Section~\ref{sec:exp:other_properties}), we find that incorporating approximate empirical Fisher information~\cite{yu2020salvaging, kirkpatrick2017overcoming} or symmetrized KL divergence~\cite{jiang2019smart} does not improve the performance over the simple $L_2$ regularized objective, while adding non-trivial computational overhead.

\paragraph{Remark (Relation to FedProx).} We note that the $L_2$ term in \ditto bears resemblance to FedProx, a method which was developed to address heterogeneity in federated optimization~\citep{li2018federated}. However, \ditto fundamentally differs from FedProx in that the goal is to learn \textit{personalized} models $v_k$, while FedProx produces a single global model $w$. For instance, when the regularization hyperparameter is zero, \ditto reduces to learning separate local models, whereas FedProx would reduce to FedAvg. In fact, \ditto is significantly more general than FedProx in that FedProx could be used as the global model solver in \ditto to optimize $G(\cdot)$. As discussed above, other regularizers beyond the $L_2$ norm may also be used in practice.

\subsection{\ditto Solver} \label{sec:solver}

To solve~\ref{obj:multitask}, we propose jointly solving for the global model  $w^*$ and personalized models $\{v_k\}_{k\in [K]}$ in an alternating fashion, as summarized in Algorithm~\ref{alg:1}. 
Optimization proceeds in two phases: (i) updates to the global model, $w^*$, are computed across the network, and then (ii) the personalized models $v_k$ are fit on each local device.  The process of optimizing $w^*$ is exactly the same as optimizing for any objective $G(\cdot)$ in federated settings: If we use iterative solvers, then at each communication round, each selected device can solve the local subproblem of $G(\cdot)$ approximately (Line 5). For personalization, device $k$ solves the global-regularized local objective $\min_{v_k} h_k(v_k; w^t)$ inexactly at each round (Line 6). Due to this alternating scheme, our solver can scale well to large networks, as it does not introduce additional communication or privacy overheads compared with existing solvers for $G(\cdot)$. 
In our experiments (all except Table~\ref{table:ditto+robust_baseline}), we use FedAvg as the objective and solver for $G(\cdot)$, under which we simply let device $k$ run local SGD on $F_k$ (Line 5). We provide a simplified algorithm definition using FedAvg for the $w^*$ update in Algorithm~\ref{alg:1_fedavg} in the appendix.  %

\begin{algorithm}
\SetAlgoLined
\DontPrintSemicolon
\SetNoFillComment
\setlength{\abovedisplayskip}{3pt}
\setlength{\belowdisplayskip}{3pt}
\setlength{\abovedisplayshortskip}{3pt}
\setlength{\belowdisplayshortskip}{3pt}
\begin{tikzpicture}[remember picture, overlay]
        \draw[line width=0pt, draw=red!30, rounded corners=2pt, fill=red!30, fill opacity=0.3]
            ([xshift=0pt,yshift=3pt]$(pic cs:a) + (150pt,6pt)$) rectangle ([xshift=-5pt,yshift=0pt]$(pic cs:b)+(110pt,-5pt)$);
\end{tikzpicture}
\KwIn{$K$, $T$, $s$, $\lambda$, $\eta$, $w^0$, $\{v^0_k\}_{k \in [K]}$}
\caption{\ditto for Personalized FL}
\label{alg:1}
    \For{$t=0, \cdots, T-1$}{
        Server randomly selects a subset of devices $S_t$, and sends $w^t$ to them\; 
        \For {device $k \in S_t$ in parallel}{
        Solve the local sub-problem of $G(\cdot)$ inexactly starting from $w^t$ to obtain $w_k^t$:
        \begin{equation*}
            w_k^t \leftarrow \textsc{update\_global}(w^t, \nabla F_k(w^t))
        \end{equation*} 
        \begin{tikzpicture}[remember picture, overlay]
        \draw[line width=0pt, draw=red!30, rounded corners=2pt, fill=red!30, fill opacity=0.3]
            ([xshift=0pt,yshift=3pt]$(pic cs:a) + (169pt,6pt)$) rectangle ([xshift=-5pt,yshift=0pt]$(pic cs:b)+(-5pt,-31pt)$);
        \end{tikzpicture}\texttt{/*~~Solve} $h_k(v_k; w^t)$\texttt{~~*/}\;
        Update $v_k$ for $s$ local iterations: 
        \begin{equation*}
            v_k = v_k - \eta (\nabla F_k(v_k) + \lambda (v_k - w^t))
        \end{equation*}
        Send $\Delta_k^t := w_k^t - w^t$ back\;
        }
        Server aggregates $\{\Delta_k^t\}$:
        \begin{equation*}
            w^{t+1} \leftarrow \textsc{aggregate} \left(w^t, \{\Delta_k^t\}_{k \in \{S_t\}}\right)
        \end{equation*}\;
        \vspace{-0.1in}
    }
    \begin{tikzpicture}[remember picture, overlay]
        \draw[line width=0pt, draw=red!30, rounded corners=2pt, fill=red!30, fill opacity=0.3]
            ([xshift=0pt,yshift=3pt]$(pic cs:a) + (135pt,6pt)$) rectangle ([xshift=-5pt,yshift=0pt]$(pic cs:b)+(37pt,-4pt)$);
    \end{tikzpicture}
    \Return{$\{v_k\}_{k \in [K]}$ (personalized), $w^T$ (global)} \;
\end{algorithm}

We note that another natural choice to solve the \ditto objective is to first obtain $w^*$, and then for each device $k$, perform finetuning on the local objective $\min_{v_k} h_k(v_k;w^*)$. These two approaches will arrive at the same solutions in strongly convex cases. In non-convex settings, we observe that there may be additional benefits of joint optimization: Empirically, we find that the updating scheme tends to guide the optimization trajectory towards a better solution compared with finetuning starting from $w^*$, particularly when $w^*$ is corrupted by adversarial attacks (Section~\ref{sec:exp:other_properties}). Intuitively, under training-time attacks, the global model may start from a random one, get optimized, and gradually become corrupted as training proceeds~\cite{{li2020gradient}}. In these cases, feeding in \textit{early} global information (i.e., before the global model converges to $w^*$) may be helpful under strong  attacks.

We note that \textcolor{black}{\ditto with joint optimization requires the devices to maintain local states (i.e., personalized models) and carry these local states to the next communication round where they are selected. Solving \ditto with finetuning does not need devices to be stateful, while losing the benefits of alternate updating discussed above.}

\paragraph{Modularity of \ditto.} From the \ref{obj:multitask} objective and Alg~\ref{alg:1}, we see that a key advantage of \ditto is its modularity, i.e., that we can readily use prior art developed for the \ref{obj:global} along with the personalization add-on of $h_k(v_k;w^*)$, as highlighted in red. This has several benefits:

\begin{itemize}[leftmargin=*]
\setlength{\itemsep}{-1pt}
    \item \textit{Optimization:} 
It is possible to plug in other methods beyond FedAvg~\citep[e.g.,][]{li2018convergence,karimireddy2020scaffold,reddi2020adaptive} in Algorithm~\ref{alg:1} to update the global model, and inherit the convergence benefits, if any (we make this more precise in Theorem~\ref{thm:convg}). 
\item \textit{Privacy:}  \ditto communicates the same information over the network as typical FL solvers for the global objective, thus preserving whatever privacy or communication benefits exist for the global objective and its respective solver. \textcolor{black}{This is different from most other personalization methods where global model updates depend on local parameters, which may raise privacy concerns~\citep{london2020pac}.}
\item \textit{Robustness:} Beyond the inherent robustness benefits of personalization, robust global methods can be used with \ditto to further improve performance (see Section~\ref{sec:exp:other_properties}). 
\end{itemize}
\vspace{-.1in}

In particular, while not the main focus of our work, we note that \ditto may offer a better \textit{privacy-utility} tradeoff than training a global model. For instance, when training \ditto, if we fix the number of communication rounds and add the same amount of noise per round to satisfy differential privacy, \ditto consumes exactly the same privacy budget as normal global training, while yielding higher accuracy via personalization (Section~\ref{sec:exp}). Similar benefits have been studied, e.g., via finetuning strategies~\cite{yu2020salvaging}.

\paragraph{Convergence of Algorithm~\ref{alg:1}.} 
Note that optimizing the global model $w^t$ does not depend on any personalized models $\{v_k\}_{k \in [K]}$. Therefore, $w$ enjoys the same global convergence rates with the solver we use for $G$. Under this observation, 
we present the local convergence of Algorithm~\ref{alg:1}.

\begin{theorem}[Local Convergence of Alg.~\ref{alg:1}; formal statement and proof in Theorem~\ref{thm:w_t_v_t}]
\label{thm:convg}
Assume for $k \in [K]$, $F_k$ is strongly convex and smooth, {under common assumptions}, if $w^t$  converges to $w^*$ with rate $g(t)$, then there exists a constant $C$$<$$\infty$ such that for  $\lambda \in \mathbb{R},$ and for $k\in [K]$, $v_k^t$ converges to $v_k^* := \argmin_{v_k} h_k(v_k; w^*)$ with rate $C g(t)$. 
\vspace{-.05in}
\end{theorem}

Using Theorem~\ref{thm:convg}, we can directly plug in previous convergence analyses for any $G(\cdot)$. 
For instance, when the global objective and its solver are those of FedAvg, we can obtain an $O(1/t)$ convergence rate for \ditto under suitable conditions (Corollary~\ref{coro:convergence_stochastic}). We provide a full theorem statement and proof of convergence in Appendix~\ref{app:convg}.

\subsection{Analyzing the Fairness/Robustness Benefits of \ref{obj:multitask} in Simplified Settings} \label{sec:theory}

In this section, we more rigorously explore the fairness/robustness benefits of \ditto on a class of linear problems. Throughout our analysis, we assume $G(\cdot)$ is the standard objective in FedAvg~\cite{mcmahan2017communication}. 

\paragraph{Point Estimation.} To provide intuition, we first examine a toy one-dimensional point estimation problem. Denote the underlying models for the devices as $\{v_k\}_{k \in [K]}$, $v_k \in \mathbb{R}$, and let the points on device $k$, $\{x_{k, 1}, \dots, x_{k, n}\}$\footnote{For ease of notation, we assume each device has the same number of training samples. It is straightforward to extend the current analysis to allow for varying number of samples per device.}, be observations of $v_k$ with random perturbation, i.e., $x_{k,i}  = v_k + z_{k, i}$, where $z_{k, i} \sim \mathcal{N}(0, \sigma^2)$ and are IID. Assume $v_k \sim \mathcal{N} (\theta, \tau^2)$, where $\theta$ is drawn from the uniform uninformative 
prior on $\mathbb{R},$ and $\tau$ is a known constant. 
Here, $\tau$ controls the degree of relatedness of the data on different devices: $\tau$=$0$ captures the case where the data on all devices are identically distributed while $\tau \to \infty$ results in the scenario where the data on different devices are completely unrelated. 
The local objective is $\min_{v_k} F_k(v_k) =  \frac{1}{2} (v_k - \frac{1}{n_k} \sum_{i=1}^{n_k} x_{k, i})^2$. In the presence of adversaries, we look at a specific type of label poisoning attack. Let $K_a$ denote the number of malicious devices, and the `capability' of an adversary is modeled by $\tau_a$, i.e., the underlying model of an adversary follows $\mathcal{N} (\theta, \tau_a^2)$ where $\tau_a^2 > \tau^2$.

We first derive the Bayes estimator (which will be the most accurate and robust) for the real model distribution by observing a finite number of training points. Then, we show that by solving \ditto, we are able to recover the Bayes estimator with a proper $\lambda^*$ (with the knowledge of $\tau$). In addition, \textit{the same} $\lambda^*$ results in the most fair solution among the set of solutions of \ditto parameterized by $\lambda$. 
This shows that \ditto with a proper choice of $\lambda$ is Bayes optimal for this particular problem instance.
\textcolor{black}{In general, in Theorem~\ref{thm:lambda_star_pe_adv_accuracy} (appendix), we prove that
\begin{align*}
    \lambda^* = \frac{\sigma^2}{n} \frac{K}{K\tau^2+\frac{K_a}{K-1} (\tau_a^2-\tau^2)}.
\end{align*}
We see that $\lambda^*$ decreases when (i) there are more local samples $n$,  (ii) the devices are less related (larger $\tau$), or (iii) the attacks are stronger (larger number of attackers, $K_a,$ and more powerful adversaries, $\tau_a$).}
Related theorems (Theorem~\ref{thm:lambda_star_pe_clean_accuracy}-\ref{thm:lambda_star_pe_adv_fairness}) are presented in Appendix~\ref{app:theory:pe}.

\begin{figure}[h!]
    \centering
    \includegraphics[width=0.5\textwidth]{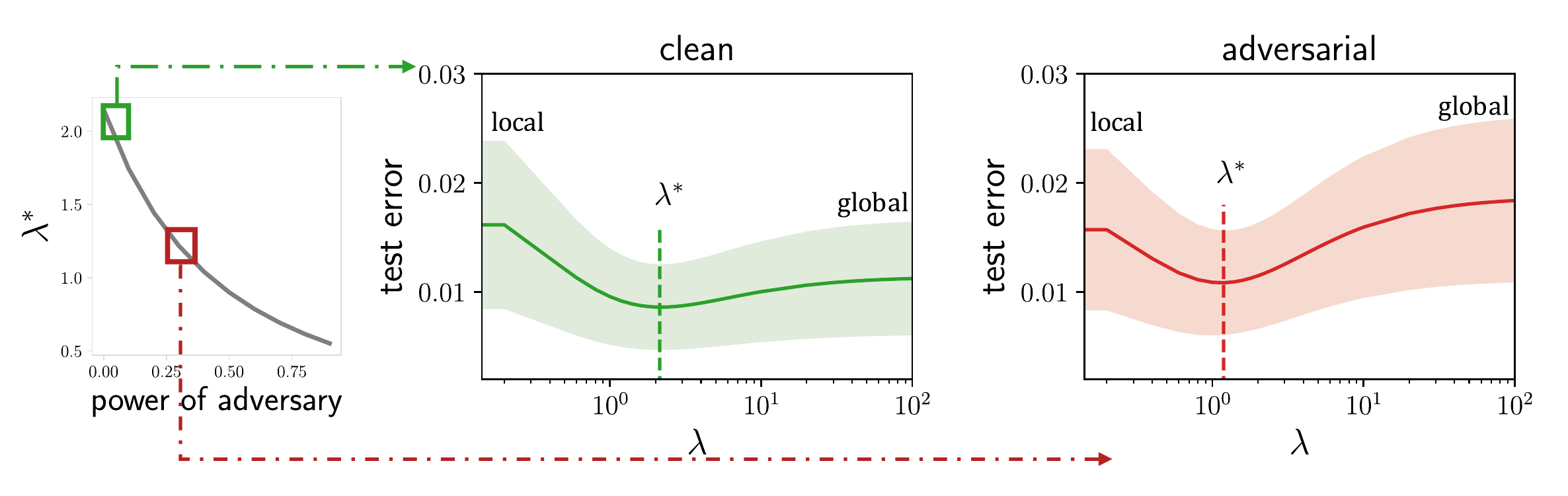}
    \caption{Empirically, the $\lambda^*$ given by Theorem~\ref{thm:lambda_star_pe_clean_accuracy}-\ref{thm:lambda_star_pe_adv_fairness} results in the most accurate, fair, and robust solution within \ditto's solution space. $\lambda^*$ is also optimal in terms of accuracy and robustness among any possible federated estimation algorithms.}
    \label{fig:pe}
\end{figure}

\begin{figure}[h!]
    \centering
    \includegraphics[width=0.5\textwidth]{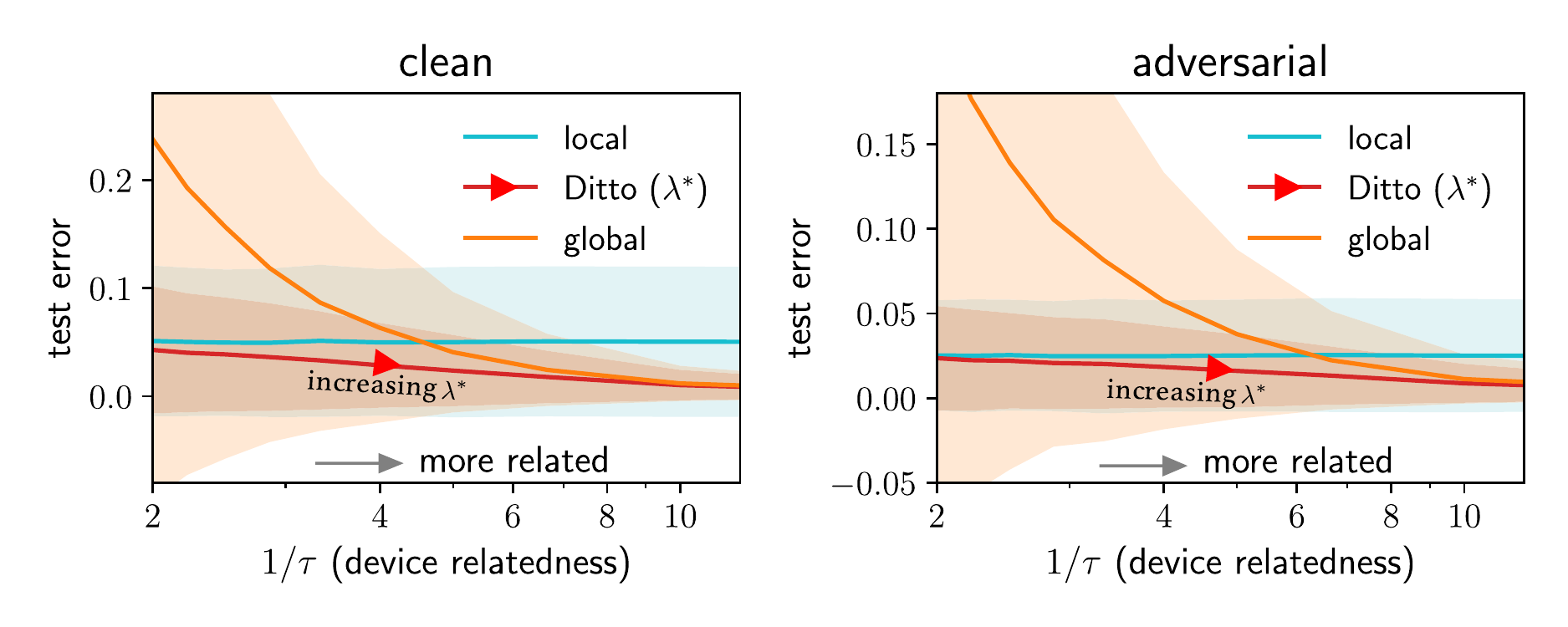}
    \caption{Impact of data relatedness across all devices. When $1/\tau$ is small (less related), local outperforms global; when $1/\tau$ is large (more related), global is better than local. \ditto ($\lambda^*$) achieves the lowest test error and variance (measured across benign devices). }
    \label{fig:pe_2}
\end{figure}

In Figure~\ref{fig:pe}, we plot average test error, fairness (standard deviation shown as error bars), and robustness (test error in the adversarial case) across a set of $\lambda$'s for both clean and adversarial cases. We see that in the solution space of \ditto, there exists a specific $\lambda$ which minimizes the average test error and standard deviation across all devices \textit{at the same time}, which is equal to the optimal $\lambda^*$ given by our theory. Figure~\ref{fig:pe_2} shows (i) \ditto with $\lambda^*$ is superior than learning local or global models, and (ii) $\lambda^*$ should increase as the relatedness between devices ($1/\tau$) increases.

\paragraph{Linear Regression.} 
All results discussed above can be generalized to establish the optimality of \ditto on a class of linear regression problems (with additional assumptions on feature covariance). We defer readers to Appendix~\ref{app:theory:lr} for full statements and proofs. While our analyses here are limited to a simplified set of attacks and problem settings, we build on this intuition in  Section~\ref{sec:exp}---empirically demonstrating the accuracy, robustness, and fairness benefits of \ditto using both convex and non-convex models, across a range of federated learning benchmarks, and under a diverse set of  attacks.

\section{Experiments} \label{sec:exp}

In this section, we first demonstrate that \ditto can inherently offer similar or superior robustness relative to strong robust baselines (Section~\ref{sec:exp:robust}). We then show it results more fair performance than recent fair methods (Section~\ref{sec:exp:fair}). \ditto is particularly well-suited for mitigating the tension between these constraints and achieving both fairness and robustness simultaneously (Section~\ref{sec:exp:tension}). We explore additional beneficial properties of \ditto in Section~\ref{sec:exp:other_properties}.

\paragraph{Setup.} For all experiments, we measure robustness via test accuracy, and fairness via test accuracy variance (or standard deviation), both across benign devices (see Def.~\ref{def:robustness},~\ref{def:fairness}).
We use datasets from common FL benchmarks~\cite{caldas2018leaf,smith2017federated,TFF}, which cover both vision and language tasks, and convex and non-convex models. Detailed datasets and models are provided in Table~\ref{table: data} in Appendix~\ref{app:exp:detail}. We split local data on each device into train/test/validation sets randomly, and measure performance on the test data. For each device, we select $\lambda$ locally based on its local validation data.
\textcolor{black}{We further assume the devices can make a binary decision on whether the attack is strong or not. For devices with very few validation samples (less than 4), we use a fixed small $\lambda$ ($\lambda$=0.1) for strong attacks, and use a fixed relatively large $\lambda$ ($\lambda$=1) for all other attacks. For devices with more than 5 validation data points, we let each select $\lambda$ from $\{0.05, 0.1, 0.2\}$ for strong attacks, and select $\lambda$ from $\{0.1, 1, 2\}$ for all other attacks. See Appendix~\ref{app:exp:full:lambda} for details. More advanced tuning methods are left for future work.}
Our code, data, and experiments are publicly available at \href{https://github.com/litian96/ditto}{\texttt{github.com/litian96/ditto}}.

\begin{figure}[b!]
    \centering
    \begin{subfigure}{0.225\textwidth}
    \includegraphics[width=\textwidth,trim=10 10 10 10]{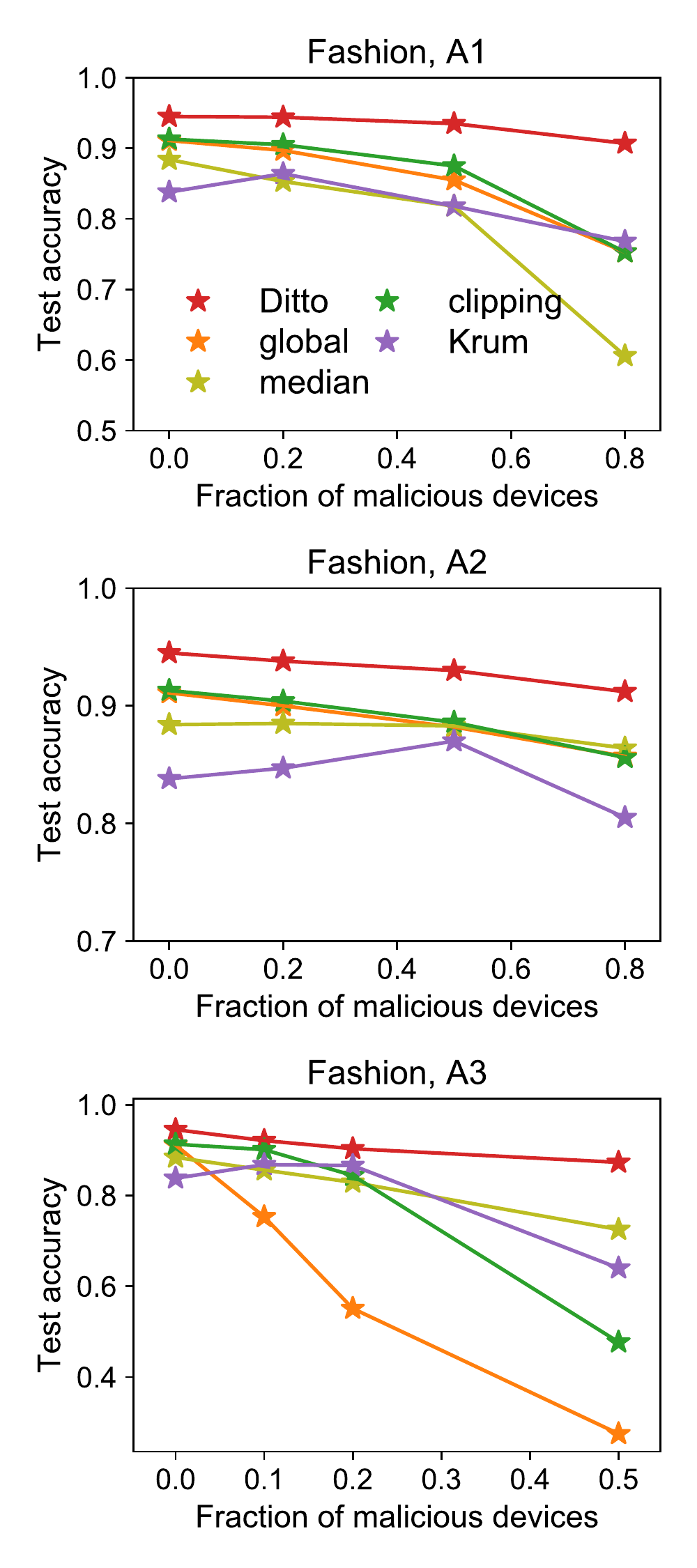}
    \end{subfigure}
    \hfill
    \begin{subfigure}{0.225\textwidth}
    \includegraphics[width=\textwidth, trim=10 10 10 10, clip]{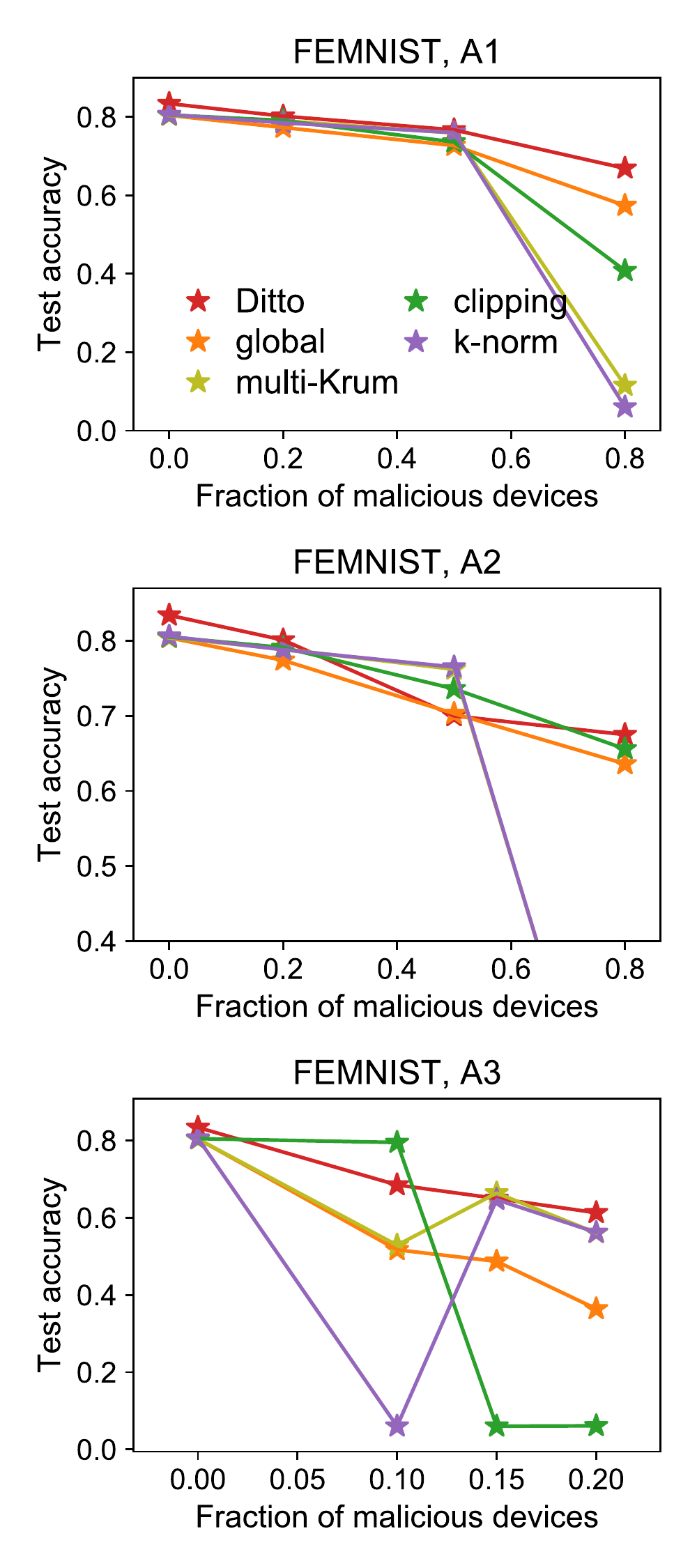}
    \end{subfigure}
    \caption{Robustness, i.e., average test accuracy on benign devices (Definition~\ref{def:robustness}), on Fashion MNIST and FEMNIST. We compare \ditto with learning a global model and three strong defense mechanisms (see  Appendix~\ref{app:exp:full} for results on all defense baselines), and find that \ditto is the most robust under almost all attacks.}
    \label{fig:mtl_robustness}
\end{figure}

\subsection{Robustness of \ditto} \label{sec:exp:robust}

Following our threat model described in Definition~\ref{def:robustness}, we apply three attacks to  corrupt a random subset of devices. We pick corruption levels until a point where there is a significant performance drop when training a global model. We compare robustness (Def.~\ref{def:robustness}) of \ditto  with various defense baselines,  presenting the results of three strongest defenses in Figure~\ref{fig:mtl_robustness}. Execution details and full results are reported in Appendix~\ref{app:exp:full:big_table}. 
As shown in Figure~\ref{fig:mtl_robustness}, \ditto achieves the highest accuracy under most attacks, particularly those with a large fraction of malicious devices. 
On average across all datasets and attacks, \ditto results in $\sim$6\% absolute accuracy improvement compared with the strongest robust baseline (Appendix~\ref{app:exp:full:big_table}).
In scenarios where a robust baseline outperforms \ditto, we have also found that replacing the global objective and its solver (FedAvg) with a robust version (e.g., using robust aggregators)  can further improve  \ditto, yielding superior performance (Section~\ref{sec:exp:other_properties}).

\setlength{\tabcolsep}{3pt}
\begin{table*}[t!]
	\caption{\textbf{Average (standard deviation)} test accuracy to benchmark performance and fairness (Definition~\ref{def:fairness}) on Fashion MNIST and FEMNIST. \ditto is either (i) more fair compared with the baselines of training a global model, or (ii) more accurate than the fair baseline under a set of attacks. We {bold} the method with {highest average minus standard deviation} across all methods.}
 	\vspace{1em}
	\centering
	\label{table:mtl_fairness}
	\scalebox{0.94}{
	\begin{tabular}{l cccccccccc} 
	   \toprule[\heavyrulewidth]
	     \textbf{Fashion} & & \multicolumn{3}{c}{{\bf A1} (ratio of adversaries)}  &  \multicolumn{3}{c}{{\bf A2} (ratio of adversaries)} & \multicolumn{3}{c}{{\bf A3} (ratio of adversaries)}\\
        \cmidrule(r){3-11}
        Methods  &  clean & 20\% & 50\% & 80\%  & 20\% & 50\% & 80\% & 10\% & 20\%  & 50\%  \\
        \midrule
        global	    &  .911 {\small (.08)}	& .897 {\small (.08)}	& .855 {\small(.10)} & {\small .753 (.13)}	& .900 {\small (.08)}	& .882 {\small (.09)} & .857 {\small (.10)}	& .753 {\small (.10)}	& .551 {\small (.13)} & .275 {\small (.12)}   \\
        local 	    & .876 {\small(.10)}	& .874 {\small (.10)}	& .876 {\small (.11)} & .879 {\small (.10)}	& .874 {\small (.10)}	& .876 {\small (.11)} & .879 {\small (.10)}	& .877 {\small (.10)}	& .874 {\small (.10)} & \textbf{.876 {\small (.11)}}  \\
        fair (TERM, $t$=1)	&  .909 {\small (.07)}	& .751 {\small (.12)}	& .637 {\small (.13)} & .547 {\small (.11)}	& .731 {\small(.13)}	& .637 {\small (.14)} & .635 {\small (.14)}	& .653 {\small (.13)}	& .601 {\small (.12)} & .131 {\small (.16)} \\
        \rowcolor{myred}
        \ditto	&  \textbf{.943 {\small (.06)}} & \textbf{.944 {\small (.07)}}  & \textbf{.937 {\small (.07)}} & \textbf{.907 {\small (.10)}} & \textbf{.938  {\small(.07)}} & \textbf{.930 {\small (.08)}} & \textbf{.913 {\small (.09)}} & \textbf{.921 {\small (.09)}} & \textbf{.902 {\small (.09)}} & {.873 {\small(.11)}} \\
	    \hline
        \hline
        \textbf{FEMNIST} & & \multicolumn{3}{c}{{\bf A1} (ratio of adversaries)}  &  \multicolumn{3}{c}{{\bf A2} (ratio of adversaries)} & \multicolumn{3}{c}{{\bf A3} (ratio of adversaries)}\\
        \cmidrule(r){3-11}
        Methods &  clean & 20\%  & 50\% & 80\% & 20\% & 50\% & 80\% & 10\% & 15\%  & 20\% \\
        \midrule
        global	    & .804 {\small (.11)}	& .773 {\small (.11)} & .727 {\small (.12)}	& .574 {\small (.15)}	& .774 {\small (.11)}	& \textbf{.703 {\small (.14)}} & .636 {\small (.15)}	& .517 {\small (.14)} & .487 {\small (.14)}	& .314 {\small (.13)} \\
        local 	    & .628 {\small (.15)}	& .620 {\small (.14)} & .627 {\small (.14)}	& .607 {\small(.14)}	& .620 {\small(.14)} & .627 {\small(.14)}	& .607 {\small(.14)}	& {.622 {\small(.14)}}	& .621 {\small(.14)} & \textbf{.620 {\small(.14)}} \\
        fair (TERM, $t$=1)	    & .809 {\small(.11)}	& .636 {\small(.15)} & .562 {\small(.13)} 	& .478  {\small(.12)}	& .440 {\small(.15)} & .336 {\small(.12)}	& .363 {\small(.12)}	& .353 {\small(.12)}  & .316 {\small(.12)} & .299 {\small(.11)} \\
        \rowcolor{myred}
        \ditto	& \textbf{.834 {\small(.09)}}	&  \textbf{.802 {\small(.10)}}	& \textbf{.762 {\small(.11)}} & \textbf{.672 {\small(.13)}} & \textbf{.801 {\small(.09)}}	& {.700 {\small(.15)}} & \textbf{.675 {\small(.14)}}	& \textbf{.685 {\small(.15)}} &  \textbf{.650 {\small(.14)}} & {.613 {\small(.13)}} \\
    \bottomrule[\heavyrulewidth]
	\end{tabular}}
\vspace{-0.1in}
\end{table*}

\subsection{Fairness of \ditto} \label{sec:exp:fair}
To explore the fairness of \ditto, we compare against TERM~\cite{li2020tilted} as a baseline. It is an improved version of the $q$-FFL~\cite{li2019fair} objective, which has been recently proposed for fair federated learning. TERM also recovers AFL~\cite{mohri2019agnostic}, another fair FL objective, as a special case. TERM uses a parameter $t$ to offer flexible tradeoffs between fairness and accuracy. 
In Table~\ref{table:mtl_fairness}, we compare the proposed objective with global, local, and fair methods (TERM) in terms of test accuracies and standard deviation. When the corruption level is high, `global' or `fair' will even fail to converge. \ditto results in more accurate and fair solutions both with and without attacks. On average across all datasets, \ditto reduces variance across devices by $\sim$10\% while improving absolute test accuracy by $5\%$ compared with TERM (on clean data).

\subsection{Addressing Competing Constraints} \label{sec:exp:tension}
In this section, we examine the competing constraints between robustness and fairness. 
When training a single global model, fair methods 
aim to encourage a more uniform performance distribution, 
but may be highly susceptible to training-time attacks in statistically heterogeneous environments. We investigate the test accuracy on benign devices when learning global, local, and fair models. In the TERM objective, we set $t=1, 2, 5$ to achieve different levels of fairness (the higher, the fairer). We perform the data poisoning attack (A1 in Def.~\ref{def:robustness}). The results are plotted in Figure~\ref{fig:fair}. 
As the corruption level increases, we see that fitting a global model becomes less robust. Using fair methods will be more susceptible to attacks. When $t$ gets larger, the test accuracy gets lower, an indication that the fair method is overfitting to the corrupted devices relative to the global baseline. 

\begin{figure}[h]
    \centering
    \begin{subfigure}{0.235\textwidth}
    \includegraphics[width=\textwidth]{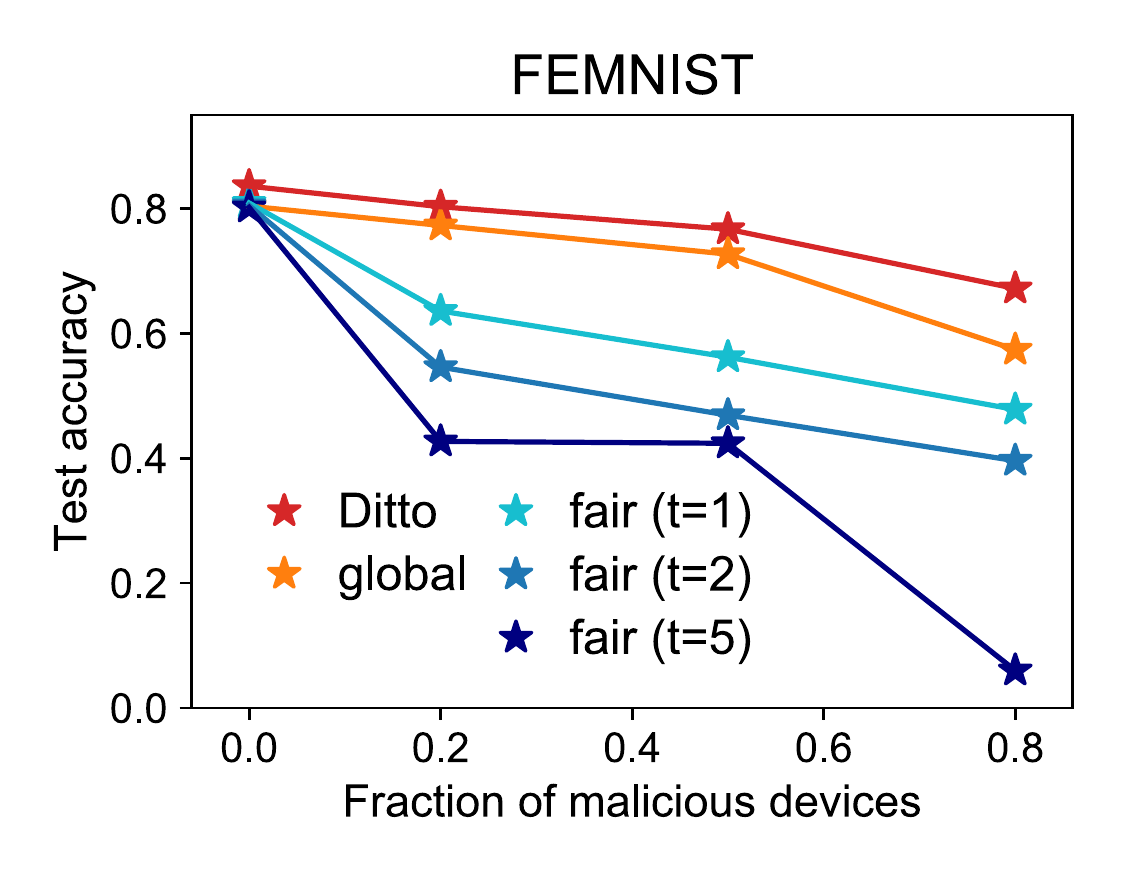}
    \end{subfigure}
    \begin{subfigure}{0.235\textwidth}
    \includegraphics[width=\textwidth]{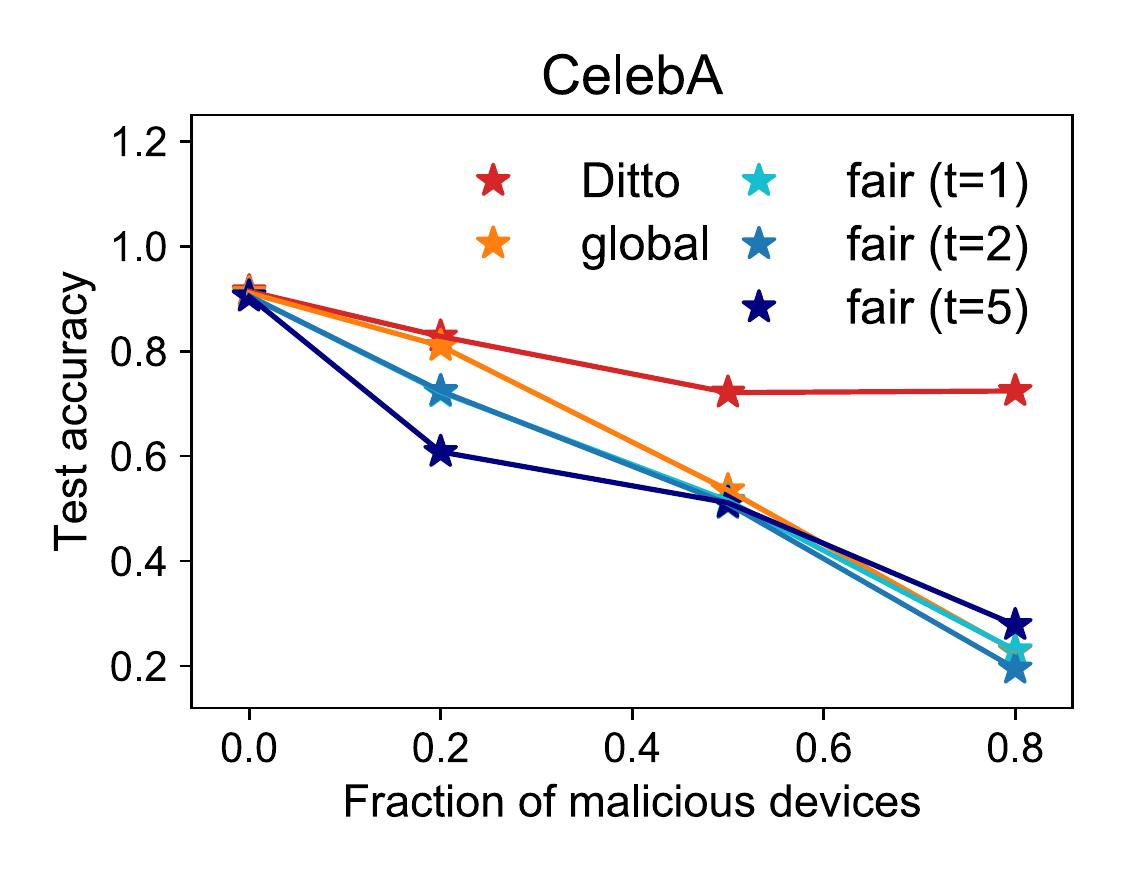}
    \end{subfigure}
    \caption{Fair methods can overfit to corrupted devices (possibly with large training losses) by imposing more weights on them, thus being particularly susceptible to attacks.}
   \vspace{0.15in}
    \label{fig:fair}
\end{figure}

\begin{figure}[h!]
    \centering
    \includegraphics[width=0.49\textwidth]{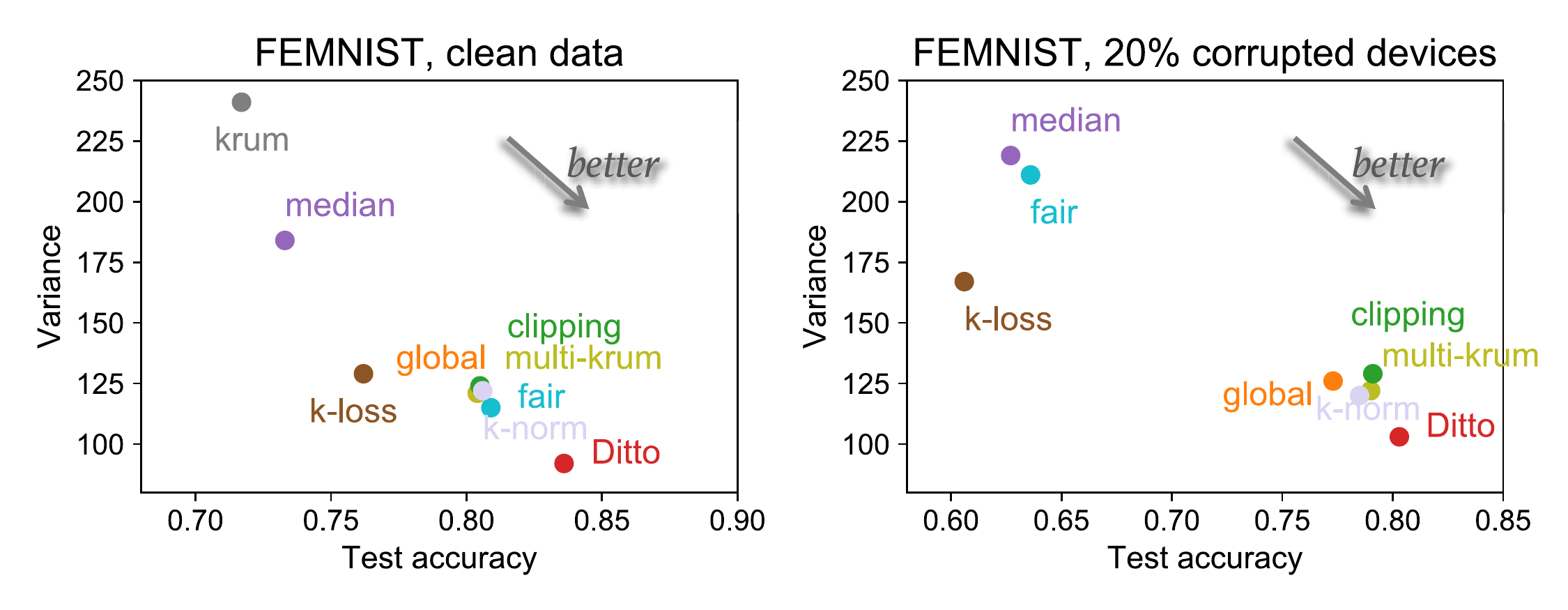}
    \vspace{-0.2in}
    \caption{Compared with learning a global model, robust baselines (i.e., the methods listed in the figure excluding `global' and `Ditto') are either robust but not fair (with higher accuracy, larger variance), or not even robust (with lower accuracy). \ditto lies at the lower right corner, which is our preferred region.}
    \vspace{-0.1in}
    \label{fig:robust}
\end{figure}

Next, we apply various strong robust methods under the same attack, and explore the robustness/accuracy and fairness performance. The robust approaches include: Krum, multi-Krum~\cite{Blanchard2017MachineLW}, taking the coordinate-wise median of gradients (`median'), gradient clipping (`clipping'), filtering out the gradients with largest norms (`k-norm'), and taking the gradient of the $k$-th largest loss where $k$ is the number of malicious devices (`k-loss'). For Krum, multi-Krum,  $k$-norm, and $k$-loss, we assume that the server knows the expected number of malicious devices that are selected each round, and can set $k$ accordingly for $k$-norm and $k$-loss. From Figure~\ref{fig:robust}, we see that robust baselines are either (i) more robust than global but less fair, or (ii) fail to provide robustness due to heterogeneity. \ditto is more robust, accurate, and fair.

\subsection{Additional Properties of \ditto} \label{sec:exp:other_properties}

\paragraph{Personalization.}
We additionally explore the performance of other personalized FL methods in terms of accuracy and fairness, on both clean and adversarial cases.
In particular, we consider objectives that (i) regularize with the average (L2SGD~\cite{hanzely2020federated}) or the learnt device relationship matrix (MOCHA~\cite{smith2017federated}), (ii) encourage closeness to the global model in terms of some specific function behavior (EWC~\cite{kirkpatrick2017overcoming, yu2020salvaging} and  Symmetrized KL (SKL)), (iii) interpolate between local and global models (APFL~\cite{deng2020adaptive} and mapper~\cite{mansour2020three}), and (iv) have been motivated by meta-learning (Per-FedAvg (HF)~\cite{fallah2020personalized}). We provide a detailed description in Appendix~\ref{app:exp:detail}.

We compare \ditto with the above alternatives, using the same learning rate tuned on FedAvg on clean data for all methods except  Per-FedAvg, which requires additional tuning to prevent divergence. 
For finetuning methods (EWC and SKL), we finetune on each local device for 50 epochs starting from the converged global model. We report results of baseline methods using their best hyperparameters.  Despite \ditto's simplicity, in Table~\ref{table: compare_other_mtl} below, we see that \ditto achieves similar or superier test accuracy with slightly lower standard deviation compared with these recent personalization methods. 

\textcolor{black}{We also evaluate the performance of MOCHA with a convex SVM model in Table~\ref{table:vehicle_full} in the appendix. MOCHA is more robust and fair than most baselines, which is in line with our reasoning that personalization can provide benefits for these constraints.}  Further understanding the robustness/fairness benefits of other personalized approaches would be an interesting direction of future work. 
\begin{table}[h!]
	\caption{\ditto is competitive with or outperforms other recent personalization methods. We report the {average (standard deviation)} of test accuracies across all devices to capture performance and fairness (Definition~\ref{def:fairness}), respectively. }
	\vspace{1em}
	\centering
	\label{table: compare_other_mtl}
	\scalebox{0.83}{
	\begin{tabular}{l cc|cc} 
	   \toprule[\heavyrulewidth]
	    & \multicolumn{2}{c}{{Clean}}  &  \multicolumn{2}{c}{50\% Adversaries (A1)} \\
        \cmidrule(r){2-5} 
        Methods &  \textbf{~FEMNIST~}  & \textbf{~CelebA~} & \textbf{~FEMNIST~}  & \textbf{~CelebA~} \\
        \hline
        global                & .804 {\small (.11)}   & .911 {\small  (.19)} & .727 {\small (.12)} & .538 {\small (.28)} \\
        local & .628 {\small (.15)}  & .692 {\small (.27)} & .627 {\small (.14)}  & .682 {\small (.27)} \\
        plain finetuning & .815 (.09) & .912 (.18) & .734 (.12) & .721 (.28) \\
        L2SGD & .817 {\small (.10)}  & .899 {\small (.18)} & .732 {\small (.15)}  & .725 {\small (.25)}  \\
        EWC  & .810 {\small (.11)}	 & .910 {\small (.18)} & .756 {\small (.12)}  & .642 {\small (.26)} \\
        SKL & .820 {\small (.10)} & \textbf{.915 {\small (.16)}} & .752 {\small (.12)}  & .708 {\small (.27)} \\
        Per-FedAvg (HF) & .827 {\small (.09)} & .907 {\small (.17)} & .604 {\small (.14)}  & \textbf{.756 {\small (.26)}} \\
        mapper & .792 (.12) & .773 (.25) & .726 (.13) & .704 (.27) \\
        APFL   & .811 {\small (.11)} & .911 {\small (.17)} & .750 {\small (.11)} & .710 {\small  (.27)} \\
        \ditto & \textbf{.836  {\small (.10)}}  & {.914 {\small (.18)}} & \textbf{.767 {\small (.10)}}  & .721 {\small  (.27)} \\
    \bottomrule[\heavyrulewidth]
	\end{tabular}}
\end{table}

\paragraph{Augmenting with Robust Baselines.}
\ditto allows the flexibility of learning robust $w^*$ leveraging any previous robust aggregation techniques, which could further improve the performance of personalized models. For instance, in the aggregation step at the server side (Line 7 in Algorithm~\ref{alg:1}), instead of simply averaging the global model updates as in FedAvg, we can aggregate them via multi-Krum, or after gradient clipping. As is shown in Table~\ref{table:ditto+robust_baseline}, \ditto combined with clipping yields improvements compared with vanilla \ditto. We present full results on different datasets trying varying robust methods in Table~\ref{table:ditto+robust_baseline_full} in the appendix.

\setlength{\tabcolsep}{4pt}
\begin{table}[h]
    \vspace{-0.05in}
	\caption{Augmenting \ditto with robust baselines can further improve performance. }
	\vspace{1em}
	\centering
	\label{table:ditto+robust_baseline}
	\scalebox{0.95}{
	\begin{tabular}{l cccccc} 
	   \toprule[\heavyrulewidth]
        \textbf{FEMNIST} & \multicolumn{2}{c}{{\bf A1}}  & \multicolumn{2}{c}{{\bf A2}} & \multicolumn{2}{c}{{\bf A3}} \\
        \cmidrule(r){2-7}
         Methods &  20\%  & 80\%  & 20\%  & 80\%  &  10\%  &  20\%   \\
        \midrule
        global & .773 & .574 & .774 & .636 & .517 & .364 \\
        clipping	        &  .791 & .408 & .791	&  .656	& .795		&  .061	 \\
        \ditto	    & .803	& \textbf{.669} & .792   & .681	    & .695		&  .650 \\
        \ditto + clipping & \textbf{.810} & .645 & \textbf{.808}	& \textbf{.684}	& \textbf{.813}	&  \textbf{.672} \\
    \bottomrule[\heavyrulewidth]
	\end{tabular}}
\end{table}

\begin{figure}[h]
    \centering
    \begin{subfigure}{0.24\textwidth}
    \includegraphics[width=\textwidth]{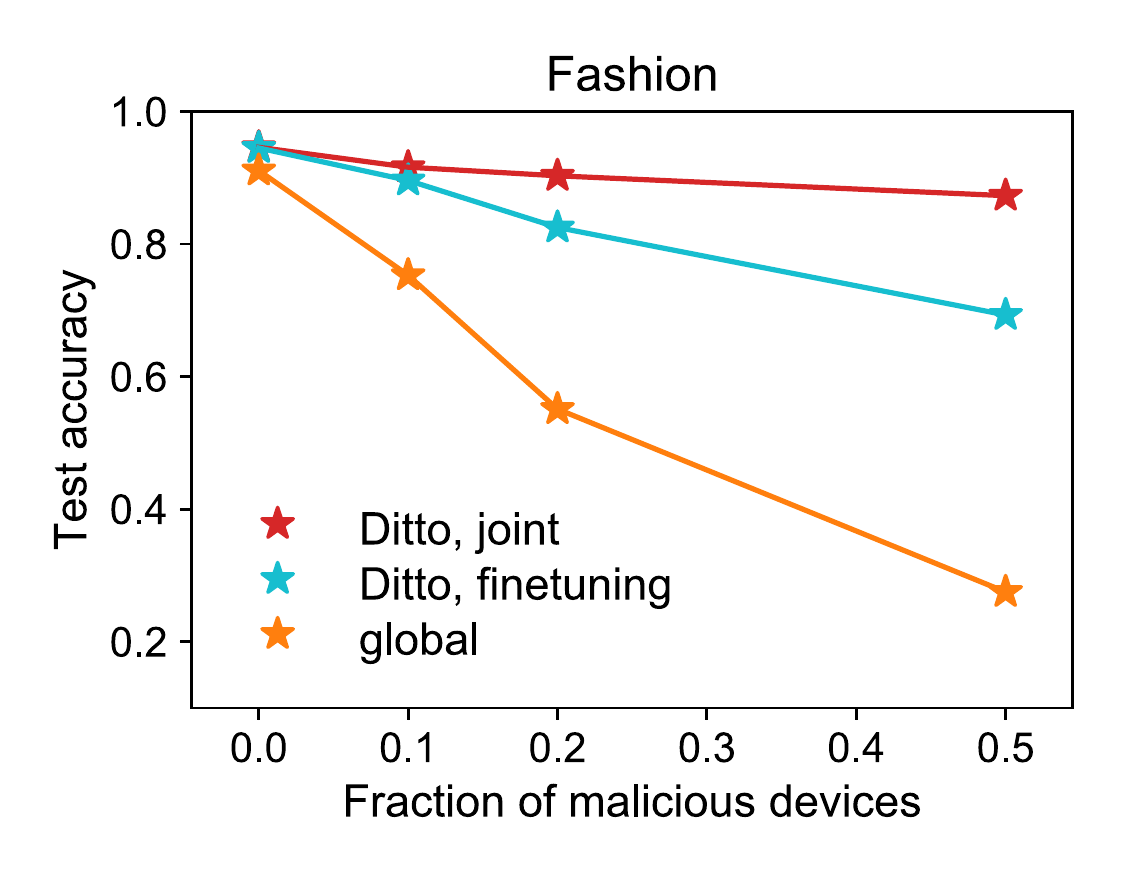}
    \end{subfigure}
    \hfill
    \begin{subfigure}{0.23\textwidth}
    \includegraphics[width=\textwidth]{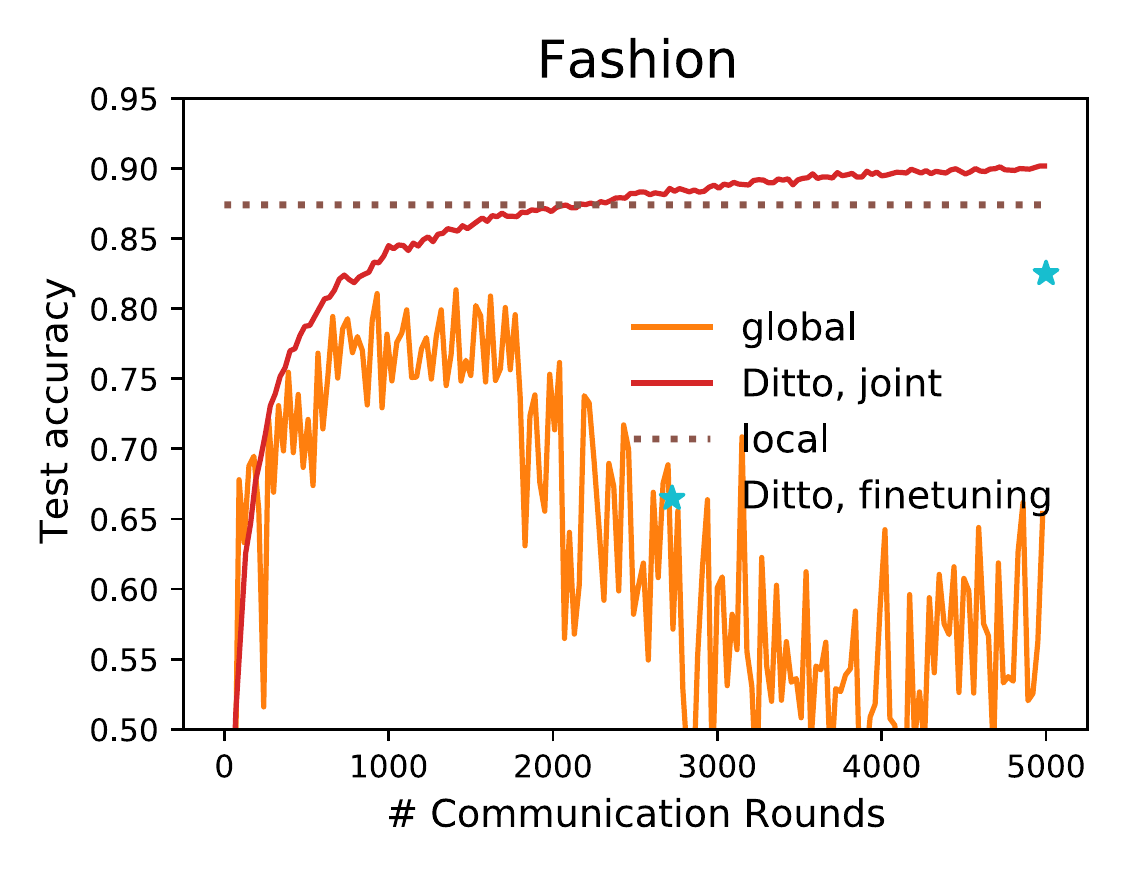}
    \end{subfigure}
    \vspace{-0.05in}
    \caption{\small \ditto with joint optimization (Algorithm~\ref{alg:1}) outperforms the alternative local finetuning solver under the strong model replacement attack.}
    \label{fig:finetuning}
\end{figure}

\paragraph{Comparing Two Solvers.}
As mentioned in Section~\ref{sec:solver}, another way to solve \ditto is to finetune on $\min_{v_k} h_k (v_k; w^*)$ for each $k \in [K] $ after obtaining $w^*$. 
We examine the performance of two solvers under the model replacement attack (A3) with 20\% adversaries. 
In realistic federated networks, it may be challenging to determine how many iterations to finetune for, particularly over a heterogeneous network of devices. 
To obtain the best performance of finetuning, we solve $\min_{v_k} h_k(v_k;w^*)$ on each device by running different iterations of mini-batch SGD and pick the best one. As shown in Figure~\ref{fig:finetuning}, the finetuning solver improves the performance compared with learning a global model, while~\ref{obj:multitask} combined with joint optimization performs the best. One can also perform finetuning after early stopping; however, it is essentially solving a different objective and it is difficult to determine the stopping criteria. We discuss this in more detail in Appendix~\ref{app:exp:full:two_solver}.

\section{Conclusion and Future Work}
We propose \ditto, a simple MTL framework, to address the competing constraints of accuracy, fairness, and robustness in federated learning. \ditto can be thought of as a lightweight personalization add-on for any global federated objective, which maintains the privacy and communication efficiency of the global solver.
We theoretically analyze the ability of \ditto to mitigate the tension between fairness and robustness on a class of linear problems. Our empirical results demonstrate that \ditto can result in both more robust and fairer models compared with strong baselines across a diverse set of attacks. Our work suggests several interesting directions of future study, such as exploring the applicability of \ditto to other attacks such as backdoor attacks~\citep[e.g.,][]{Sun2019CanYR}; understanding the fairness/robustness properties of other personalized methods; and considering additional constraints, such as privacy.

\section*{Acknowledgements}
The work of TL, SH, and VS was supported in part by the
National Science Foundation Grant IIS1838017, a Google Faculty Award, a Facebook Faculty Award, and the CONIX Research Center. Any opinions, findings, and conclusions or
recommendations expressed in this material are those of the author(s) and do not necessarily reflect
the National Science Foundation or any other funding agency.

\bibliography{references}
\bibliographystyle{icml2021}

\newpage
\appendix
\onecolumn

\section*{Appendix}
{
}
We provide a simple table of contents below for easier navigation of the appendix.

\textbf{CONTENTS}

\textbf{Section~\ref{app:theory}: Analysis of the Federated Multi-Task Learning Objective \ditto}

\quad \quad Section~\ref{app:theory:general_property}: {Properties of \ditto for Strongly Convex Functions } 

\quad \quad Section~\ref{app:theory:lr}: {Federated Linear Regression} 

\quad \quad Section~\ref{app:theory:pe}: {The Case of Federated Point Estimation} 

\textbf{Section~\ref{app:convg}: {Algorithm and Convergence Analysis}}

\textbf{Section~\ref{app:exp:detail}: {Experimental Details}}

\quad \quad Section~\ref{app:exp:data}: {Datasets and Models} 

\quad \quad  Section~\ref{app:exp:baseline}: {Personalization Baselines} 

\textbf{Section~\ref{app:exp:full}: {Additional and Complete Experiment Results}} 

\quad \quad  Section~\ref{app:exp:full:two_solver}: {Comparing with Finetuning} 

\quad \quad  Section~\ref{app:exp:full:lambda}: {Tuning $\lambda$}

\quad \quad  Section~\ref{app:exp:full:robustify}: {\ditto Augmented with Robust Baselines} 

\quad \quad  Section~\ref{app:exp:full:big_table}: {\ditto Complete Results}

\newpage
\section{Analysis of the Federated Multi-Task Learning Objective \ditto} \label{app:theory}

Here, we provide theoretical analyses of \ref{obj:multitask}, mainly on a class of linear models. 
In this linear setting, we investigate accuracy, fairness, and robustness of \ditto. We first discuss some general properties of \ditto for strongly convex functions in terms of the training performance in Section~\ref{app:theory:general_property}. We next present our main results on characterizing the benefits (accuracy, fairness, and robustness) of \ditto on linear regression in Section~\ref{app:theory:lr}. Finally, we present results on a special case of linear regression (federated point estimation problem examined in Section~\ref{sec:theory}) in Section~\ref{app:theory:pe}.

\subsection{Properties of \ditto for Strongly Convex Functions } \label{app:theory:general_property}
Let the \ditto objective on device $k$ be
\begin{equation}
    h_k(w) = F_k(w) + \lambda \psi(w), 
\end{equation}
where $F_k$ is strongly convex, and
\begin{align}
    & \psi(w) := \frac{1}{2}\| w- w^* \|^2,\\
    & w^* := \arg\min_{w} \left\{ \frac{1}{K}\sum_{k \in [K]} F_k(w)\right\}.
\end{align}

Let
\begin{equation}
    \widehat{w}_k(\lambda)  =  \arg\min_{w} h_k(w).
\end{equation}
Without any distributional assumptions on the tasks, we first characterize the solutions of the objective $h_k(w)$.
\begin{lemma}
For all $\lambda \geq 0,$
\begin{align}
    \frac{\partial}{\partial \lambda } F_k(\widehat{w}_k(\lambda)) &\geq 0,\\
    \frac{\partial}{\partial \lambda } \psi(\widehat{w}_k(\lambda)) &\leq 0.
\end{align}
In addition, for all $k$, if $F_k(w^*)$ is finite, then
\begin{equation}
    \lim_{\lambda \to \infty}  \widehat{w}_k(\lambda) = w^*.
\end{equation}
\end{lemma}

\begin{proof}
The proof here directly follows the proof in~\citet[Theorem 3.1,][]{hanzely2020federated}.
\end{proof}

As $\lambda$ increases, the local empirical training loss $F_k(\widehat{w}_k(\lambda))$ will also increase, and the resulting personalized models will be closer to the global model. 
Therefore, $\lambda$ effectively controls how much personalization we impose. Since for any device $k \in [K]$, training loss is minimized when $\lambda=0$, training separate local models is the most robust and fair  \textit{in terms of training performance when we do not consider generalization}. 
 
However, in order to obtain the guarantees on the test performance, we need to explicitly model the joint distribution of data on all devices. In the next section, we explore a Bayesian framework on a class of linear problems to examine the generalization, fairness, and robustness of the \ditto objective, all on the underlying test data.

\newpage
\subsection{Federated Linear Regression} \label{app:theory:lr}

We first examine the case without corrupted devices in Section~\ref{app:theory:mtl_clean}. We prove that there exists a $\lambda$ that results in an optimal average test performance among all possible federated learning algorithms, which coincides with the optimal $\lambda$ in \ditto's solution space in terms of fairness. When there are adversaries, we analyze the robustness benefits of \ditto in Section~\ref{app:theory:mtl_adversary}. In particular, we show there exists a $\lambda$ which leads to the highest test accuracy across benign devices (i.e., the most robust) \text{and} minimizes the variance of the test error across benign devices (i.e., the most fair) jointly.

Before we proceed, we first state a technical lemma that will be used throughout the analyses. 
\begin{lemma}
 \label{lem:parallel-general}
 Let $\theta$ be drawn from the non-informative uniform prior on $\mathbb{R}^d$. Further, let $\{\phi_k\}_{k \in [K]}$ denote noisy observations of $\theta$ with additive zero-mean independent Gaussian noises with covariance matrices $\{ \Sigma_k \}_{k \in [K]}$. Let 
\begin{equation}
 \Sigma_\theta := \left(\sum_{k \in [K]} \Sigma^{-1}_k\right)^{-1}.
\label{eq:Sigma-theta}
\end{equation}
Then, conditioned on $\{\phi_k\}_{k \in [K]}$, we can write $\theta$ as
 $$
 \theta = \Sigma_\theta  \sum_{k \in [K]} \Sigma_k^{-1} \phi_k + z,
 $$
 where $z$ is $\mathcal{N}(0,  \Sigma_\theta)$ which is independent of $\{\phi_k\}_{k \in [K]}$.
\end{lemma}

Lemma~\ref{lem:parallel-general} is a generalization of Lemma 11 presented in~\citet{mahdavifar2017global} (restated in Lemma~\ref{lem:parallel} below) to the multivariate Gaussian case. The proof also follows from the proof in ~\citet{mahdavifar2017global}.

 \begin{lemma}[Lemma 11 in~\citet{mahdavifar2017global}]
 \label{lem:parallel}
 Let $\theta$ be drawn from the non-informative uniform prior on $\mathbb{R}$. Further, let $\{\phi_k\}_{k \in [K]}$ denote noisy observations of $\theta$ with additive zero-mean independent Gaussian noises with variances $\{ \sigma^2_k \}_{k \in [K]}$. Let 
\begin{equation}
 \frac{1}{\sigma^2_\theta} := \sum_{k \in [K]} \frac{1}{\sigma^2_k}.
\label{eq:sigma-theta}
\end{equation}
Then, conditioned on $\{\phi_k\}_{k \in [K]}$, we can write $\theta$ as
 $$
 \theta = \sigma^2_\theta  \sum_{k \in [K]} \frac{\phi_k}{\sigma^2_k}+ z,
 $$
 where $z$ is $\mathcal{N}(0,  \sigma^2_\theta)$ which is independent of $\{\phi_k\}_{k \in [K]}$. 
\end{lemma}

\subsubsection{No Adversaries: \ditto for Accuracy and Fairness} \label{app:theory:mtl_clean}

We consider a Bayesian framework.
Let $\theta$ be drawn from the non-informative prior on $\mathbb{R}^d$, i.e., uniformly distributed on $\mathbb{R}^d$.
We assume that $K$ devices have their data distributed with parameters $\{w_k\}_{k \in [K]}$:
\begin{equation}
    w_k = \theta + \zeta_k,
\end{equation}
where $\zeta_k \sim \mathcal{N}(0, \tau^2 \mathbf{I}_d)$ are I.I.D, and $\mathbf{I}_d$ denotes the $d\times d$ identity matrix.  $\tau$ controls the degree of dependence between the tasks on different devices. If $\tau = 0,$ then the data on all devices is distributed according to parameter $\theta,$ i.e., the tasks are the same, and if $\tau \to \infty,$ the tasks on different devices become completely unrelated.

We first derive optimal estimators $\{w_k\}_{k \in [K]}$ for each device $w_k$ given observations $\{X_i, y_i\}_{i \in [K]}$.

\begin{lemma} \label{lemma:w_hat_LR}
Assume that we have
\begin{equation}
    y = Xw + z
\end{equation}
where $y \in \mathbb{R}^n$, $X\in \mathbb{R}^{n\times d}$, and $w \in \mathbb{R}^d$, and $z \in \mathbb{R}^n.$ Further assume that $z \sim \mathcal{N}(0, \sigma^2 \mathbf{I}_d)$ and $w$ follows the non-informative uniform prior on $\mathbb{R}^d$. Let
\begin{equation}
    \widehat{w} = (X^TX)^{-1} X^T y.
\end{equation}
Then, we have $\widehat{w}$ follows a multi-variate normal distribution as follows:
\begin{equation}
    \widehat{w} \sim \mathcal{N}\left( (X^TX)^{-1}X^T y, \sigma^2 (X^T X)^{-1} \right).
\end{equation}
\end{lemma}

\begin{lemma}\label{lemma:minus_k}
Let 
\begin{equation}
    \widehat{w}_i : = (X_i^TX_i)^{-1} X_i^T y_i.
\end{equation}

Let 
\begin{equation}
    \Sigma_i := \sigma^2 (X_i^T X_i)^{-1} + \tau^2 \mathbf{I}_d.
\end{equation}
Further, let
\begin{equation}
 \Sigma_{\theta}^{\setminus k} := \left(\sum_{i \in [K], i\neq k} \Sigma^{-1}_i\right)^{-1}.
\end{equation}
Further let
\begin{equation}
    \mu_\theta^{\setminus k}  := \Sigma_{\theta}^{\setminus k}  \sum_{i \in [K], i \neq k} \Sigma_i^{-1} \widehat{w}_i 
\end{equation}
Then, conditioned on $\{X_i, y_i\}_{i \in [K], i \neq k}$, we can write $\theta$ as
 $$
 \theta = \mu_\theta^{\setminus k}  + \eta,
 $$
 where $\eta$ is $\mathcal{N}(0,  \Sigma_{\theta}^{\setminus k})$ which is independent of $\{X_i, y_i\}_{i \in [K], i \neq k}$.
\end{lemma}

\begin{proof}
From Lemma~\ref{lemma:w_hat_LR}, we know $\widehat{w}_i$ is a noisy observation of the underlying $w_i$ with additive covariance $\sigma^2 (X_i^T X_i)^{-1}$. For $\{w_k\}_{k \in [K]}$ defined in our setup, $\widehat{w}_i$ is a noisy observation of $\theta$ with additive zero mean and covariance $\Sigma_i := \tau^2 \mathbf{I}_d + \sigma^2 (X_i^T X_i)^{-1}$. 
The proof completes by applying Lemma~\ref{lem:parallel-general} to $\{\widehat{w}_i\}_{i \in [K], i \neq k}$.
\end{proof}

\begin{lemma}\label{lemma:w_k_clean}
Let 
\begin{equation}
    \Sigma_{w_k}^{\setminus k} := \Sigma_{\theta}^{\setminus k} + \tau^2 \mathbf{I}_d.
\end{equation}
Further, let 
\begin{equation}
    \Sigma_{w_k} := \left( (\Sigma_{w_k}^{\setminus k})^{-1} + (\Sigma_k - \tau^2 \mathbf{I}_d)^{-1}\right)^{-1}.
\end{equation}
Conditioned on $\{X_i, y_i\}_{i \in [K]},$ we have
\begin{equation} \label{eq:optimal_w_k}
    w_k = \Sigma_{w_k}  (\Sigma_k-\tau^2 \mathbf{I}_d)^{-1} \widehat{w}_k + \Sigma_{w_k}  (\Sigma_{w_k}^{\setminus k})^{-1} \mu_\theta^{\setminus k} + \zeta_k,
\end{equation}
where $\zeta_k \sim \mathcal{N}(0, \Sigma_{w_k})$.
\end{lemma}

\begin{proof}
$\widehat{w}_k$ is a noisy observation of $w_k$ with additive noise with zero mean and covariance $\sigma^2 (X_k^T X_k)^{-1}$ (which is $\Sigma_k - \tau^2 \mathbf{I}_d$). From Lemma~\ref{lemma:minus_k}, we know conditioned on $\{X_i, y_i\}_{i \in [K], i \neq k}$, $\mu_{\theta}^{\setminus k}$ is a noisy observation of $\theta$ with covariance $\Sigma_{\theta}^{\setminus k}$. Hence, with respect to $w_k$, the covariance is $\Sigma_{\theta}^{\setminus k} + \tau^2 \mathbf{I}_d := \Sigma_{w_k}^{\setminus k}$. The conclusion follows by applying Lemma~\ref{lem:parallel-general} to $\widehat{w}_k$ and $\mu_{\theta}^{\setminus k}$.
\end{proof}

Let the empirical loss function of the linear regression problem on device $k$ be
\begin{align}
    F_k(w) = \frac{1}{n} \left \|X_k w - y_k\right\|^2.
\end{align}
Then the estimator $\widehat{w}_k$ is $(X_k^T X_k)^{-1} X^T y_k$.
Applying the previous lemmas, we obtain an optimal estimator $w_k$ given all training samples from $K$ devices (see~\eqref{eq:optimal_w_k}). $w_k$ is Bayes optimal among all solutions that can be achieved by {any} learning method. Next, we examine the \ditto objective and its solution space parameterized by $\lambda$.

Let each device solve the following objective
\begin{align}
    \min_w h_k(w) = F_k(w) + \frac{\lambda}{2} \left\|w - w^*\right\|^2,~\text{s.t.} \quad w^* = \frac{1}{K}\argmin_w \sum_{k=1}^K F_k(w).
\end{align}

The local empirical risk minimizer for each device $k$ is
\begin{align}
    \widehat{w}_k(\lambda) &= \left(\frac{1}{n} X_k^\top X_k + \lambda I\right)^{-1} \left(\frac{1}{n}X_k^\top Y_k + \lambda w^*\right) \\
    &= \left(\frac{1}{n}X_k^\top X_k + \lambda I\right)^{-1} \left(\left(\frac{1}{n} X_k^\top X_k\right) \widehat{w}_k + \lambda \sum_{k=1}^K (X^\top X) ^{-1} X_k^\top X_k\widehat{w}_k\right) \label{eq:w_k_lambda}
\end{align}

We next prove that for any $k \in  [K]$, $\widehat{w}_k(\lambda)$ with a specific $\lambda$ can achieve the optimal $w_k$.

\begin{theorem} \label{thm:lambda_star_lr_clean_accuracy}
Assume for any $1\leq i \leq K$, $X_i^T X_i = \beta \mathbf{I}_d$ for some constant $\beta$. Let $\lambda^*$ be the optimal $\lambda$ that minimizes the test performance on device $k$, i.e.,
\begin{align}
    \lambda^* = \argmin_{\lambda} E\left\{\left. F_k(\widehat{w}_k(\lambda)) \right| \widehat{w}_k , \mu_{\theta}^{\setminus k}\right\}.
\end{align}
Then,
\begin{align}
    \lambda^* = \frac{\sigma^2}{n\tau^2}.
\end{align}
\end{theorem}

\begin{proof}

Notice that
\begin{align}
    \argmin_{\lambda} E\left\{F_k(\widehat{w}_k(\lambda)) | \widehat{w}_k , \mu_{\theta}^{\setminus k}\right\} 
    &=  \argmin_{\lambda} E\left\{ \|X_k \widehat{w}_k(\lambda) - (X_k w_k + z_k)\|^2 | \widehat{w}_k , \mu_{\theta}^{\setminus k}\right\} \\ 
    &= \argmin_{\lambda} E\left\{ \|X_k \left(\widehat{w}_k(\lambda) - w_k\right)\|^2 | \widehat{w}_k , \mu_{\theta}^{\setminus k}\right\}  \\
    &= \argmin_{\lambda} E \left\{ \left\|w_k - \widehat{w}_k (\lambda)\right\|^2 | \widehat{w}_k , \mu_{\theta}^{\setminus k} \right \}.
\end{align}
Plug in $X_k^T X_k = \beta \mathbf{I}$ into~\eqref{eq:optimal_w_k} and \eqref{eq:w_k_lambda} respectively, we have the optimal estimator $w_k$ is
\begin{align}
    w_k = \left(\frac{K-1}{\frac{\sigma^2}{\beta} + K \tau^2} + \frac{\beta}{\sigma^2}\right)^{-1} \frac{\beta}{\sigma^2} \widehat{w}_k +  \left(\frac{K-1}{\frac{\sigma^2}{\beta} + K \tau^2} + \frac{\beta}{\sigma^2}\right)^{-1} \frac{\beta}{\sigma^2 + K\tau^2 \beta} \sum_{i \in [K], i \neq k}\widehat{w}_i + \zeta_k,
\end{align}
and $\widehat{w}_k(\lambda)$ is
\begin{align}
     \widehat{w}_k(\lambda) = \left(\frac{n}{\beta + n\lambda}\right) \left(\left(\frac{\beta}{n} + \frac{\lambda}{K}\right) \widehat{w}_k + \frac{\lambda}{K} \sum_{i \in [K], i \neq k} \widehat{w}_i\right).
\end{align}
Taking $w_k$ and $\widehat{w}_k(\lambda)$ into
\begin{align}
    \lambda^* = \argmin_{\lambda} E \left\{ \left. \|w_k - \widehat{w}_k (\lambda)\right \|_2^2|\mu_{\theta}^{\setminus k}, \widehat{w}_k \right \}
\end{align}
gives $\lambda^* = \frac{\sigma^2}{n\tau^2}$, as $\widehat{w}_k(\lambda^*)$ is the MMSE estimator of $w_k$ given the observations.
\end{proof}

\begin{remark}\label{remark:optimality_lambda}
We note that by using  $\lambda^*$ in \ditto, we not only achieve the most accurate solution for the objective, but also we achieve the most accurate solution of any possible federated linear regression algorithm in this problem, as \ditto with $\lambda^*$ realizes the MMSE estimator for $w_k$.
\end{remark}
We have derived an optimal $\lambda^* = \frac{\sigma^2}{n\tau^2}$ for \ditto in terms of generalization. Recall that we define fairness as the variance of the performance across all devices~\cite{hashimoto2018fairness, li2019fair}. Next, we prove that the same $\lambda^*$ that minimizes the expected MSE also achieves the optimal fairness among all \ditto solutions. 

\begin{theorem}\label{thm:lambda_star_lr_clean_fairness}
Assume for any $1\leq i \leq K$, $X_i^T X_i = \beta \mathbf{I}_d$ for some constant $\beta$. Among all possible solutions \ditto parameterized by $\lambda$, $\lambda^*$ results in the most fair performance across all devices when there are no adversaries, i.e., it minimizes the variance of test performance (test loss) across all devices.
\end{theorem}

\begin{proof}
Denote the variance of test performance (loss) across $K$ devices as $\var_K \left\{\|X_k \widehat{w}_k(\lambda) - y_k\|_2^2\right\}$. Let
\begin{align}
    \widehat{E}_k \{a_k\} := \frac{1}{K} \sum_{k \in [K]} a_k.
\end{align}
Then
\begin{align}
    \argmin_{\lambda} \var_K \left\{\|X_k \widehat{w}_k(\lambda) - y_k\|_2^2\right\} &=   \argmin_{\lambda} \var_K \left\{\|X_k \widehat{w}_k(\lambda) - (X_k w_k + z_k) \|_2^2\right\}  \\ &=   \argmin_{\lambda}  \var_K \left\{\|X_k (\widehat{w}_k(\lambda) - w_k)\|_2^2\right\} \\ &= \argmin_{\lambda}  \var_K \left\{\|\widehat{w}_k(\lambda) - w_k\|_2^2\right\}  \\
     &= \argmin_{\lambda}  \widehat{E}_K \left\{ \left(\|w_k - \widehat{w}_k\|_2^2 \right)^2\right\} - \left(\widehat{E}_K \left\{\|w_k - \widehat{w}_k (\lambda)\|_2^2\right\}\right)^2.
\end{align}
Note that
\begin{align}
    w_k - \widehat{w}_k(\lambda) = \zeta + a_k,
\end{align}
where
\begin{align}
    a_k = \widehat{w}_k(\lambda^*) - \widehat{w}_k(\lambda),
\end{align}
and $\lambda^* = \frac{\sigma^2}{n\tau^2}$.

We have
\begin{align}
      & \widehat{E}_K \left\{ \left(\|w_k - \widehat{w}_k\|_2^2 \right)^2\right\} - \left(\widehat{E}_K \left\{\|w_k - \widehat{w}_k (\lambda)\|_2^2\right\}\right)^2 \\ &=  \widehat{E}_K \left\{ \left(\sum_{i}^d ({w_k}_i - \widehat{w}_k(\lambda)_i)^2\right)^2\right\} - \left(\widehat{E}_K \left\{\sum_{i}^d ({w_k}_i - \widehat{w}_k(\lambda)_i)^2\right\}\right)^2 \\
      &= \widehat{E}_K \left\{ \left(\sum_{i}^d (\zeta_i + a_{ki})^2\right)^2\right\} - \left(\widehat{E}_K \left\{\sum_{i}^d (\zeta_i + a_{ki})^2\right\}\right)^2,
\end{align}
where ${w_k}_i$, $\widehat{w}_k(\lambda)_i$, $\zeta_i$, and $a_{ki}$ denotes the $i$-th dimension of $w_k$, $\widehat{w}_k(\lambda)$, $\zeta$, and $a_k$ and $d$ is the model dimension.

We next expand the variance by decomposing it into two parts.
We note
\begin{align}
    &\widehat{E}_K \left\{ \left(\sum_{i}^d (\zeta_i + a_{ki})^2\right)^2\right\} - \left(\widehat{E}_K \left\{\sum_{i}^d (\zeta_i + a_{ki})^2\right\}\right)^2 \\ &= \sum_{i}^d \widehat{E}_k  \left\{(\zeta_i + a_{ki})^4\right\} - \sum_{i}^d \left(\widehat{E}_K \left\{(\zeta_i + a_{ki})^2\right\} \right)^2 \label{eq:part1} \\
    & + 2 \sum_{i, j \in [d], i \neq j}  \widehat{E}_K \left\{\left(\zeta_i + {a_k}_i\right)^2\left(\zeta_j + {a_k}_j\right)^2\right\} -  2 \sum_{i, j \in [d], i \neq j}  \widehat{E}_K \left\{\left(\zeta_i + {a_k}_i\right)^2 \right\} \widehat{E}_K \left\{\left(\zeta_j + {a_k}_j\right)^2\right\}. \label{eq:part2}
\end{align}

For any $i \in [d]$, we have
\begin{align}
& E\left\{ \left.  \widehat{E}_K \left\{ (\zeta_i + a_{ki})^4\right\} - \left( \widehat{E}_K \left\{ (\zeta_i + a_{ki} )^2\right\}\right)^2 \right|\mu_{\theta}^{\setminus k}, \widehat{w}_k \right\}\\
    & = E\left\{ \left.  \widehat{E}_K \left\{ \zeta^4_i + 6\zeta_i^2 a_{ki}^2 +a_{ki}^4\right\} - \left( \widehat{E}_K \left\{ \zeta_i^2 + a_{ki}^2\right\}\right)^2 \right|\mu_{\theta}^{ \setminus k}, \widehat{w}_k \right\}\\
    & = E\left\{ \left.  \widehat{E}_K \left\{ \zeta_i^4 + 6\zeta_i^2 a_{ki}^2 +a_{ki}^4\right\} - 
    \left( \widehat{E}_K \left\{ \zeta_i^2 \right\}\right)^2
    - 2 \widehat{E}_K \left\{ \zeta_i^2 \right\}
    \widehat{E}_K \left\{  a_{ki}^2\right\}
    -
    \left( \widehat{E}_K \left\{  a_{ki}^2\right\}\right)^2 \right|\mu_{\theta}^{\setminus k}, \widehat{w}_k \right\}\\
     & =  3\sigma_w^4 + 6\sigma_w^2 \widehat{E}_K \left\{a_{ki}^2\right\} + \widehat{E}_K \left\{a_{ki}^4\right\} - 
    \sigma_w^4
    - 2 \sigma_w^2
    \widehat{E}_K \left\{ a_{ki}^2\right\}
    -
    \left( \widehat{E}_K \left\{  a_{ki}^2\right\}\right)^2 \\
     & =  2\sigma_w^4 + 4\sigma_w^2 \widehat{E}_K \left\{a_{ki}^2\right\} + \widehat{E}_K \left\{a_{ki}^4\right\} 
    -
    \left( \widehat{E}_K \left\{  a_{ki}^2\right\}\right)^2, \label{eq:part1_1}
\end{align}

where $\sigma_w$ is the $i$-th diagonal of $\Sigma_{w_k}$ which is the same across all $k$'s and all dimensions, and we have used the fact that we can swap expectations, and  $E\{\zeta_i^4\} = 3\sigma_w^4,$ given that $\zeta_i$ is Gaussian distributed and $\Sigma_{w_k}$ is a diagonal matrix.

For any $i,j \in [d], i \neq j$, we have
\begin{align}
    & E \left\{ \left. \widehat{E}_K \left(\zeta_i + a_{ki}\right)^2\left(\zeta_j + a_{kj}\right)^2 \right| \mu_{\theta}^{\setminus k}, \widehat{w}_k \right\} -  E \left\{ \left. \widehat{E}_K \left(\zeta_i + a_{ki} \right)^2  \widehat{E}_K \left (\zeta_j + a_{kj}\right)^2 \right| \mu_{\theta}^{\setminus k}, \widehat{w}_k  \right\}  \\ &=\widehat{E}_k\{a_{ki}^2 a_{kj}^2\} - \widehat{E}_k\{a_{ki}^2\} \widehat{E}_k \{ a_{kj}^2\}, \label{eq:part2_2}
\end{align}
where we have used the fact that $\Sigma_{w_k}$ is a diagonal matrix.

Plugging~\eqref{eq:part1_1} and~\eqref{eq:part2_2} into~\eqref{eq:part1} and~\eqref{eq:part2} yields
\begin{align}
& E \left\{ \left. \var_K \left\{\|\widehat{w}_k(\lambda) - w_k\|_2^2\right\} \right| \mu_{\theta}^{\setminus k}, \widehat{w}_k \right\} \\ & =
    2d\sigma_w^4+\sum_i 4\sigma_w^2\widehat{E}_k\{a_{ki}^2\}+\sum_i\widehat{E}_k\{a_{ki}^4\}-\sum_i\left(\widehat{E}_k\{a_{ki}^2\}\right)^2+2\sum_{i \neq j}\left(\widehat{E}_k\{a_{ki}^2 a_{kj}^2\} - \widehat{E}_k\{a_{ki}^2\} \widehat{E}_k \{ a_{kj}^2\}\right) \\ &= 2d\sigma_w^4+\sum_i 4\sigma_w^2\widehat{E}_k\{a_{ki}^2\} +
    \sum_i\widehat{E}_k\{a_{ki}^4\}+2\sum_{i\neq j}\widehat{E}_k \{a_{ki}^2 a_{kj}^2\}-(\sum_i \left(\mathbb{E}_k\{a_{ki}^2\}\right)^2+2\sum_{i \neq j}\widehat{E}_k \{a_{ki}^2\} \widehat{E}_k \{ a_{kj}^2)\}) \\ &= 2d\sigma_w^4+\sum_i4\sigma_w^2\widehat{E}_k\{a_{ki}^2\} +
    \widehat{E}_k\{(\sum_i a_{ki}^2)^2\}-(\sum_i\widehat{E}_k\{a_{ki}^2\})^2 \\ &= 2d\sigma_w^4+\sum_i 4\sigma_w^2\widehat{E}_k\{a_{ki}^2\} + 
    \frac{1}{K}\sum_k(\sum_ia_{ki}^2)^2-(\frac{1}{K}\sum_k\sum_i a_{ki}^2)^2 \geq 2d\sigma_w^2, \label{eq:variance_value}
\end{align}
where setting $\{a_{ki}\}_{1\leq k \leq K, 1 \leq i \leq d}=0$ achieves the minimum.
\end{proof}

\paragraph{Observations.} From the optimal $\lambda^*=\frac{\sigma^2}{n\tau^2}$ for mean test accuracy and variance of the test accuracy, we have the following observations.
\begin{itemize}[leftmargin=*]
    \item Test error and variance can be jointly minimized with one $\lambda$.
    \item As $n \to \infty,$ $\lambda^* \to 0,$ i.e., when each local device has an infinite number of samples, there is no need for federated learning, and training local models is optimal in terms of generalization and fairness.
    \item As $\tau \to \infty,$ $\lambda^* \to 0,$ i.e., if the data on different devices (the tasks) are unrelated, then training local models is optimal; On the other hand, as $\tau \to 0,$ $\lambda^* \to \infty,$ i.e., if the data across all devices are identically distributed, or equivalently if the tasks are the same, then training a global model is the best we can achieve.
\end{itemize}

So far we have proved that the same $\lambda^*$ achieves the best performance (expected mean square error) for any device $k$ \textit{and} fairness (variance of mean square error) without considering adversaries. In Section~\ref{app:theory:mtl_adversary} below, we analyze the benefits of \ditto for fairness and robustness in the presence of adversaries. 

\subsubsection{With Adversaries: \ditto for Accuracy, Fairness, and Robustness} \label{app:theory:mtl_adversary}

As a special case of data poisoning attacks defined in our threat model (Definition~\ref{def:robustness}), we make the following assumptions on the adversaries.

Let $K_a$ and $K_b \geq 1$ denote the number of malicious and benign devices, respectively, such that $K = K_a + K_b.$

\begin{definition}
We say that a device $k$ is a benign device if $w_k \sim \theta + \mathcal{N}(0, \tau^2 \mathbf{I}_d)$; and we say a device $k$ is a malicious device (or an adversary) if $w_k \sim \theta + \mathcal{N}(0, \tau_a^2 \mathbf{I}_d)$ where $\tau_a > \tau$.
\end{definition}

As mentioned in Definition~\ref{def:fairness} and~\ref{def:robustness}, in the presence of adversaries, we measure fairness as the performance variance on \textit{benign} devices, and robustness as the average performance across \textit{benign} devices. We next characterize the benefits of \ditto under such metrics.

\begin{lemma} \label{lemma:w_k_adv}
Let $w_k$ be the underlying model parameter of a benign device $k$. 
Let  
\begin{align}
    \widehat{w}_i := (X_i^T X_i)^{-1} X_i^T y_i,~i \in [K].
\end{align}
Let
\begin{align}
    \Sigma_w^{\setminus k} = \frac{1}{(K-1)^2} \left(\sum_{i \in [K_b], i \neq k} \left(\sigma^2 (X_i^T X_i)^{-1} + \tau^2 \mathbf{I}_d\right) +  \sum_{i \in [K_a], i \neq k} \left( \sigma^2 (X_i^T X_i)^{-1} + \tau_a^2 \mathbf{I}_d \right) \right), \label{eq:Sigma_w_minus_k}
\end{align}
and
\begin{align}
    \Sigma_{w,a}^{-1} = (\sigma^2 (X_k^T X_k)^{-1})^{-1} + (\Sigma_{w}^{\setminus k} + \tau^2 \mathbf{I}_d)^{-1}. \label{eq:Sigma_w_k}
\end{align}
Conditioned on observations $\widehat{w}_k$ and $\widehat{w}^{K\setminus k} := \frac{1}{K-1} \sum_{i \neq k, i \in [K]} \widehat{w}_i$, we have
\begin{align}
    w_k = \Sigma_{w, a} (\sigma^2 (X_k^T X_k)^{-1})^{-1} \widehat{w}_k + \Sigma_{w, a} (\Sigma_{w}^{\setminus k} + \tau^2 \mathbf{I}_d)^{-1} \widehat{w}^{K \setminus k} + \zeta_k,
\end{align}
where $\zeta_k \sim \mathcal{N}(0, \Sigma_{w, a})$.
\end{lemma}

\begin{proof}
For malicious devices $i \in [K_a]$ and $i \neq k$, the additive covariance of $w_i$ with respect to $\theta$ is $\sigma^2 (X_i^T X_i)^{-1} + \tau_a^2 \mathbf{I}_d$. For benign devices $i \in [K_b]$ and $i \neq K$, the covariance is $\sigma^2 (X_i^T X_i)^{-1} + \tau^2 \mathbf{I}_d$. Therefore, the covariance of $\widehat{w}^{K\setminus k}$ is $\Sigma_w^{\setminus k}$. Hence given $\widehat{w}^{K \setminus k}$, $w_k$ is Gaussian with covariance $\Sigma_w^{\setminus k} + \tau^2 \mathbf{I}_d$. $\widehat{w}^{K \setminus k}$ can be viewed as a noisy observation of $w_k$ with covariance $\Sigma_w^{\setminus k} + \tau^2 \mathbf{I}_d$. $\widehat{w}_k$ is a noisy observation of $w_k$ with covariance $\sigma^2 (X_k^T X_k)^{-1}$. The proof follows by applying Lemma~\ref{lem:parallel-general} to $\widehat{w}_k$ and $\widehat{w}^{K \setminus k}$.
\end{proof}

\begin{theorem}  \label{thm:lambda_star_lr_adv_accuracy}
Assume for any $1 \leq i \leq K$, $X_i^T X_I = \beta \mathbf{I}_d$ for some constant $\beta$. Let $k$ be a benign device. 
Let $\lambda^*_a$ be the optimal $\lambda$ that minimizes the test performance on device $k$, i.e., 
\begin{equation}
   \lambda^* = \argmin_{\lambda} E\left\{F_k( \left. \widehat{w}_k(\lambda)) \right | \widehat{w}_k , \widehat{w}^{K \setminus k} \right\}.
\end{equation}
Then,
\begin{equation} \label{eq:lambda_a_star}
    \lambda^*_a = \frac{\sigma^2}{n} \frac{K}{K\tau^2+ \frac{K_a}{K-1}(\tau^2_a-\tau^2)}.
\end{equation}
\end{theorem}

\begin{proof}
We obtain $\lambda_a^*$ following the proof of Theorem~\ref{thm:lambda_star_lr_clean_accuracy}.
\end{proof}

\begin{theorem} \label{thm:lambda_star_lr_adv_fairness}
Among all \ditto solutions parameterized by $\lambda$, $\lambda^*_a$ results in the most fair performance across all benign devices, i.e., it minimizes the variance of test performance (test mean square error) on benign devices.
\end{theorem}
\begin{proof}
Similarly, we look at the variance of the test loss across benign devices:
\begin{align}
    \argmin_{\lambda} E \left\{\var_{K_b}\left\{ \|X_k \widehat{w}_k(\lambda) - y_k\|_2^2 \right\}\right\} &=  \argmin_{\lambda} E \left\{\var_{K_b}\left\{ \|w_k(\lambda) - w_k \|_2^2 \right\}\right\} 
    \\ &= \argmin_{\lambda}  \widehat{E}_{K_b} \left\{ \left(\|w_k - \widehat{w}_k\|_2^2 \right)^2\right\} - \left(\widehat{E}_{K_b} \left\{\|w_k - \widehat{w}_k (\lambda)\|_2^2\right\}\right)^2 .
\end{align}
The rest of the proof is the same as the proof of Theorem~\ref{thm:lambda_star_lr_clean_fairness}, except that we set $a_k = \widehat{w}_k(\lambda) - \widehat{w}_k(\lambda^*_a).$
\end{proof}

\begin{remark} \label{remark:optimality_2}
For any benign device $k$, the solution we obtain by solving \ditto  with $\lambda_a^*$ is the most robust solution one could obtain among any federated point estimation method given observations $\widehat{w}_k$ and $\widehat{w}^{K \setminus k}$. $\lambda_a^*$ also results in a most fair model in the solution space of \ditto parameterized by $\lambda$.  
\end{remark}

\begin{lemma} \label{lemma:lr_error_variance_value}
The expected test error minimized at $\lambda_a^*$ is $d\sigma_{w,a}^2$; and the variance of the test loss minimized at $\lambda_a^*$ is $2d \sigma_{w,a}^4$, where $\sigma_{w,a}$ denotes the diagonal element of $\Sigma_{w,a}$.
\end{lemma}

\begin{proof}
For the expected test performance, we note that
\begin{equation}
    E\left\{ \left.\|w_k - \widehat{w}_k(\lambda_a^*)\|^2 \right| \widehat{w}^{K\setminus k}, \widehat{w}_k \right\} = E[\|\text{diag}(\Sigma_{w, k})\|^2] = d\sigma_{w,k}^2.
\end{equation}
For variance, as $a_k=0$ if $\lambda = \lambda_a^*$, from~\eqref{eq:variance_value}, we get
\begin{equation}
    \var_{K_b}\left\{ \|w_k - \widehat{w}_k(\lambda_a^*)\|^2\right\} = 2d \sigma_{w,k}^4.
\end{equation}
\end{proof}

\paragraph{Observations.} From $\lambda_a^*$, we have the following interesting observations.
\begin{itemize}[leftmargin=*]
    \item Mean test error on benign devices (robustness) and variance of the performance across benign devices (fairness) can still be minimized with the same $\lambda_a$ in the presence of adversaries. 
    \item As $\tau_a \to \infty$, $\lambda^*_a \to 0,$ i.e., training local models is optimal in terms of robustness and fairness when adversary's task may be arbitrarily far from the the task in the benign devices.
    \item As $\tau \to 0$, if $\tau_a > 0$, $\lambda^*_a < \infty$, which means that learning a global model is \emph{not} optimal even with homogeneous data in the presence of adversaries.
    \item $\lambda^*_a$ is a decreasing function of the number ($K_a$) and the capability ($\tau_a$) of the corrupted devices. In other words, as the attacks become more adversarial, we need more personalization.
    \item The smallest test error is $\sigma^2_{w, a}$, and the optimal variance is $2\sigma^4_{w,a}$, which are both increasing with $K_a$ (number of adversarial devices) or $\tau_a$ (the power of adversary) by inspecting~\eqref{eq:Sigma_w_minus_k} and~\eqref{eq:Sigma_w_k}. This reveals a fundamental tradeoff between fairness and robustness.
\end{itemize}

\paragraph{Discussion.} Through our  analysis, we prove that \ditto with an appropriate $\lambda$ is more accurate, robust, and fair compared with training global or local models on the problem described in~\ref{app:theory:lr}. We provide closed-form solutions for $\lambda^*$ across different settings (with and without adversaries), and show that \ditto can achieve fairness and robustness jointly. In the future, we plan to generalize the current theoretical framework to more general models. In the next section, we present a special case of the current analysis, a federated point estimation problem, which is also studied in Section~\ref{sec:theory} as a motivating example.

\subsection{The Case of Federated Point Estimation} \label{app:theory:pe}

We consider the one-dimensional federated point estimation problem, which is a special case of linear regression. Similarly, 
Let $\theta$ be drawn from the non-informative prior on $\mathbb{R}$. We assume that $K$ devices have their data distributed with parameters $\{w_k\}_{k \in [K]}$.
\begin{equation}
    w_k = \theta + \zeta_k,
\end{equation}
where $\zeta_k \sim \mathcal{N}(0, \tau^2)$ are IID. 

Let each device have $n$ data points denoted  by $\mathbf{x}_k = \{x_{k,1}, \ldots, x_{k,n}\},$ such that
\begin{equation}
    x_{k,i} = w_k + z_{k, i},
\end{equation}
where $z_{k,i} \sim \mathcal{N}(0, \sigma^2)$ and are IID.

Assume that 
\begin{equation}
    F_k(w) = \frac{1}{2} \left( w - \frac{1}{n} \sum_{i \in [n]}x_{k,i}\right)^2,
\end{equation}
and denote by $\widehat{w}_k$ the minimizer of the empirical loss $F_k$. It is clear that
\begin{equation}
    \widehat{w}_k = \frac{1}{n} \sum_{i \in [n]}x_{k,i}.
\end{equation}
Further, let 
\begin{equation}
    w^* := \arg\min_{w} \left\{ \frac{1}{K}\sum_{k \in [K]} F_k(w)\right\}.
\end{equation}
It is straightforward calculation to verify that
\begin{equation}
    w^* = \frac{1}{nK}\sum_{i \in [n]} \sum_{k \in [K]} x_{k,i} = \frac{1}{K} \sum_{k \in [K]}\widehat{w}_k.
\end{equation}

\begin{lemma}
Denote by $\widehat{w}_k(\lambda)$ the minimizer of $h_k.$ 
Then,
\begin{align}
    \widehat{w}_k(\lambda) & = \frac{\lambda}{1+\lambda}w^* + \frac{1}{1+\lambda} \widehat{w}_k\\
    &= \frac{\lambda}{(1+\lambda)K} \sum_{j \neq k} \widehat{w}_j + \frac{K + \lambda}{(1+\lambda)K} \widehat{w}_k.
\end{align}
\label{lem:w-k-lambda}
\end{lemma}

Let 
\begin{equation}
    \sigma^2_n := \frac{\sigma^2}{n},
\end{equation}
and
\begin{equation}
    \widehat{w}^{K\setminus k}:= \frac{1}{K-1}\sum_{j \neq k} \widehat{w}_j.
\end{equation}
\begin{lemma}
Given observations $\widehat{w}^{K\setminus k}$ and $\widehat{w}_k$, $w_k$ is Gaussian distributed and given by 
\begin{equation}
    w_k = \frac{\sigma^2_w }{\sigma^2_n} \widehat{w}_k + \frac{(K-1)\sigma^2_w}{K \tau^2 + \sigma^2_n}  \widehat{w}^{K\setminus k} + \xi,
\end{equation}
where
\begin{equation}
    \frac{1}{\sigma_w^2} = \frac{1}{\sigma^2_n} + \frac{K-1}{K\tau^2 + \sigma^2_n},
\end{equation}
and
\begin{equation}
    \xi \sim \mathcal{N}\left( 0, \sigma_w^2\right).
\end{equation}
\label{thm:wk}
\end{lemma}

\begin{proof}
The proof follows by setting $X_k = \mathbf{1}_{n\times 1}$ ($k \in [K]$) in Lemma~\ref{lemma:w_k_clean}.
\end{proof}

\begin{theorem} \label{thm:lambda_star_pe_clean_accuracy}
Let $\lambda^*$ be the optimal $\lambda$ that minimizes the test performance, i.e., 
\begin{equation}
    \lambda^* = \arg\min_{\lambda} E\left\{ \left.(w_k - \widehat{w}_k(\lambda))^2 \right| \widehat{w}^{K\setminus k}, \widehat{w}_k \right\}.
\end{equation}
\end{theorem}
Then,
\begin{equation} \label{eq:lambda_star}
    \lambda^* = \frac{\sigma^2_n}{\tau^2} = \frac{\sigma^2}{n \tau^2}.
\end{equation}
\begin{proof}
The proof follows by setting $X_k = \mathbf{1}_{n \times 1}$  ($k \in [K]$) in Theorem~\ref{thm:lambda_star_lr_clean_accuracy}.
\end{proof}

\begin{theorem} \label{thm:lambda_star_pe_clean_fairness}
Among all \ditto's solutions, $\lambda^*$ results in the most fair performance across all devices when there are no adversaries, i.e., it minimizes the variance of test performance (test mean square error).
\end{theorem}
\begin{proof}
The proof follows by setting $X_k = \mathbf{1}_{n \times 1}$ ($k \in [K]$) in Theorem~\ref{thm:lambda_star_lr_clean_fairness}.
\end{proof}

Similarly, the adversarial case presented below (including setups, lemmas, and theorems) is also a special case of the adversarial scenarios for linear regression.

Let $K_a$ and $K_b \geq 1$ denote the number of adversarial and benign devices, respectively, such that $K = K_a + K_b.$

\begin{definition}
We say that a device $k$ is a benign device if $w_k \sim \theta + \mathcal{N}(0, \tau^2)$; and we say a device $k$ is a malicious device (or an adversary) if $w_k \sim \theta + \mathcal{N}(0, \tau_a^2)$ where $\tau_a \geq \tau$.
\end{definition}

\begin{lemma}
Let $w_k$ be the parameter associated with a benign device. Given observations $\widehat{w}^{K\setminus k}:= \frac{1}{K-1}\sum_{j \neq k} \widehat{w}_j$ and $\widehat{w}_k,$ $w_k$ is Gaussian distributed and given by
\begin{equation}
    w_k = \frac{\sigma^2_{w,a} }{\sigma^2_n} \widehat{w}_k + \frac{(K-1)\sigma^2_{w,a}}{K \tau^2 + \sigma^2_n + \frac{K_a}{K-1} (\tau_a^2 - 
    \tau^2)} \widehat{w}^{K \setminus k} + \xi_a,
\end{equation}
where
\begin{equation}
    \frac{1}{\sigma_{w,a}^2} = \frac{1}{\sigma^2_n} + \frac{K-1}{K \tau^2 + \sigma^2_n + \frac{K_a}{K-1} (\tau_a^2-\tau^2)},
\label{eq:sigma_w_a}
\end{equation}
and
\begin{equation}
    \xi_a \sim \mathcal{N}\left( 0, \sigma_{w,a}^2\right).
\end{equation}
\end{lemma}

\begin{proof}
The proof follows by setting $X_k=\mathbf{1}_{n \times 1}$ ($k \in [K]$) in Lemma~\ref{lemma:w_k_adv}.
\end{proof}

\begin{theorem} \label{thm:lambda_star_pe_adv_accuracy}
Let $w_k$ be a benign device. 
Let $\lambda^*_a$ be the optimal $\lambda$ that minimizes the test performance, i.e., 
\begin{equation}
    \lambda^*_a = \arg\min_{\lambda} E\left\{ \left.(w_k - \widehat{w}_k(\lambda))^2 \right| \widehat{w}^{K\setminus k}, \widehat{w}_k \right\}.
\end{equation}
Then,
\begin{equation} \label{eq:lambda_a_star}
    \lambda^*_a = \frac{\sigma^2}{n} \frac{K}{K\tau^2+ \frac{K_a}{K-1}(\tau^2_a-\tau^2)}.
\end{equation}
\end{theorem}

\begin{proof}
The proof follows by setting $X_k = \mathbf{1}_{n \times 1}$ ($k \in [K]$) in Theorem~\ref{thm:lambda_star_lr_adv_accuracy}.
\end{proof}

\begin{theorem} \label{thm:lambda_star_pe_adv_fairness}
Among all solutions of Objective~\eqref{obj:multitask} parameterized by $\lambda$, $\lambda^*_a$ results in the most fair performance across all benign devices, i.e., it minimizes the variance of test performance (test mean square error) on benign devices.
\end{theorem}
\begin{proof}
The proof follows by setting $X_k = \mathbf{1}_{n \times 1}$ ($k \in [K]$) in Theorem~\ref{thm:lambda_star_lr_adv_fairness}.
\end{proof}

\begin{lemma}
The expected test error minimized at $\lambda_a^*$ is $\sigma_{w,a}^2$; and the variance of the test performance minimized at $\lambda_a^*$ is $2\sigma_{w,a}^4$.
\end{lemma}

\begin{proof}
The proof follows by setting $X_k = \mathbf{1}_{n \times 1}$ ($k \in [K]$) in Lemma~\ref{lemma:lr_error_variance_value}.
\end{proof}

\newpage

\section{Algorithm and Convergence Analysis}
\label{app:convg}

In this section, we first present the specific algorithm (Algorithm~\ref{alg:1_fedavg}) that we use for most of our experiments (all except for Table~\ref{table:ditto+robust_baseline} and~\ref{table:ditto+robust_baseline_full}). Algorithm~\ref{alg:1_fedavg} is a special case of the more general \ditto solver (Algorithm~\ref{alg:1}), where we use $\min_w \sum_{k \in [K]} p_k F_k(w)$ as the global objective and FedAvg as its solver.  As before, the \ditto personalization add-on is highlighted in red. In addition, we prove that personalized models can inherit the convergence rates of the optimal global model for any $G(\cdot)$ (Theorem~\ref{thm:w_t_v_t}), and provide convergence guarantees for the special case of Algorithm~\ref{alg:1_fedavg} (Corollary~\ref{coro:convergence_stochastic}).
\begin{algorithm}
\SetAlgoLined
\DontPrintSemicolon
\SetNoFillComment
\setlength{\abovedisplayskip}{0pt}
\setlength{\belowdisplayskip}{0pt}
\setlength{\abovedisplayshortskip}{0pt}
\setlength{\belowdisplayshortskip}{0pt}
\begin{tikzpicture}[remember picture, overlay]
        \draw[line width=0pt, draw=red!30, rounded corners=2pt, fill=red!30, fill opacity=0.3]
            ([xshift=0pt,yshift=3pt]$(pic cs:a) + (180pt,6pt)$) rectangle ([xshift=-5pt,yshift=0pt]$(pic cs:b)+(145pt,-5pt)$);
\end{tikzpicture}
\KwIn{$K$, $T$, $s$, $\lambda$, $\eta_g$, $\eta_l$, $w^0$, $p_k$, $\{v^0_k\}_{k \in [K]}$}
\caption{\ditto for Personalized FL in the case of $G(\cdot)$ being FedAvg~\cite{mcmahan2017communication}}
\label{alg:1_fedavg}
    \For{$t=0, \cdots, T-1$}{
        Server randomly selects a subset of devices $S_t$, and sends $w^t$ to them\;
        \For {device $k \in S_t$ in parallel}{
            Sets $w_k^t$ to $w^t$ and updates $w_k^t$ for $r$ local iterations on $F_k$:
            \begin{equation*}
                w_k^t = w_k^t - \eta_g \nabla F_k(w_k^t) 
            \end{equation*} \;
            \begin{tikzpicture}[remember picture, overlay]
            \draw[line width=0pt, draw=red!30, rounded corners=2pt, fill=red!30, fill opacity=0.3]
                ([xshift=0pt,yshift=3pt]$(pic cs:a) + (295pt,6pt)$) rectangle ([xshift=-5pt,yshift=0pt]$(pic cs:b)+(2pt,-18pt)$);
            \end{tikzpicture}
            Updates $v_k$ for $s$ local iterations: 
            \begin{equation*}
                v_k = v_k - \eta_l (\nabla F_k(v_k) + \lambda (v_k - w^t)
            \end{equation*}\;
            Sends $\Delta_k^t := w_k^t - w^t$ back\;
        }
        Server updating $w^{t+1}$ as
        \begin{align*}
            w^{t+1} \leftarrow w^t + \frac{1}{|S_t|}\sum_{k\in S_t}\Delta_k^{t}
        \end{align*}
    }
    \begin{tikzpicture}[remember picture, overlay]
            \draw[line width=0pt, draw=red!30, rounded corners=2pt, fill=red!30, fill opacity=0.3]
                ([xshift=0pt,yshift=3pt]$(pic cs:a) + (134pt,7pt)$) rectangle ([xshift=-5pt,yshift=0pt]$(pic cs:b)+(37pt,-4pt)$);
            \end{tikzpicture}
    \Return{$\{v_k\}_{k \in [K]}$ (personalized), $w^T$ (global)} \;
\end{algorithm}

To analyze the convergence behavior of Algorithm~\ref{alg:1} and~\ref{alg:1_fedavg}, we first state a list of assumptions below. 
\begin{itemize}
    \item \textcolor{black}{The global model converges with rate $g(t)$, i.e., there exists $g(t)$ such that $\lim_{t \to \infty} g(t) = 0$, $\mathbb{E}[\|w^t- w^* \|^2] \leq g(t)$.}
    \item For $k \in [K]$, $F_k$ is $\mu$-strongly convex. 
    \item The expectation of stochastic gradients is uniformly bounded at all devices and all iterations, i.e.,
    \begin{align}
        \mathbb{E}[\|\nabla F_k (w^t, \xi^t)\|^2] \leq G_1^2.
    \end{align}
\end{itemize}

Let $w^*$ be defined as 
\begin{equation}
    w^* := \min_w \, G(F_1(w), \dots\, F_K(w))
\end{equation}
i.e., $w^*$ is the empirically optimal global model for $G(\cdot)$.  Let $u_k^*$ denote the empirically optimal local model on device $k$, i.e., 
\begin{align}
    u_k^*=\argmin_u F_k(u).
\end{align}

We introduce an additional assumption on the distance between optimal local models $\{u_k^*\}_{k \in [K]}$ and the optimal global model $w^*$ below.

\begin{itemize}
    \item The $L_2$ distance between the optimal local models and the optimal global model is bounded, i.e., for $k \in [K]$,
    \begin{align}
       \|u_k^*-w^*\| \leq M. \label{eq:heterogeneity_assum}
    \end{align}
\end{itemize}
\textcolor{black}{This assumption sets an upper bound on the deviation of the local model on device $k,$ with the global model. It can in turn be viewed as boundedness of heterogeneity of the training data across devices. When local data are farther from being IID, $M$ tends to be larger. Recall that in the fairness/robustness analysis of \ditto (Appendix~\ref{app:theory}), we model the relatedness of \textit{underlying} models via $\tau$, and $\mathbb{E}[\|w_k-\theta\|^2]=d \tau^2$ where $w_k$ is the \textit{underlying} model for device $k$ and $d$ is the model dimension. $M$ is related to $\tau^2$ as
\begin{align}
    \mathbb{E}[\|u_k^*-w^*\|^2] &\leq 2\mathbb{E}[\|\mu_k^*-w_k\|^2] + 4\mathbb{E}[\|w_k-\theta\|^2]+4\mathbb{E}[\|\theta-w^*\|^2] \\ &\to 4d\tau^2.
\end{align}
when $n_k$ and the total number of samples across all devices are sufficiently large, considering the linear problems we studied. We later show that for convergence, $\lambda$ scales with $1/M$, which is consistent with $\lambda^*$ (for fairness/robustness) scaled with $1/\tau^2$. }

Further let
\begin{align}
    v_k^* = \argmin_v h_k(v; w^*),
\end{align}
i.e., $v_k^*$ is the optimal personalized model for device $k$. We are interested in the convergence of $v_k$ to $v_k^*$.
We first characterize the progress of updating personalized models for one step under a \textit{general $G(\cdot)$}.

\begin{lemma}[Progress of one step]\label{lemma:v_w_stochastic}
Under assumptions above, let device $k$ get selected with probability $p_k$ at each communication round, with decaying local step-size $\frac{2}{(t+1)(\mu+\lambda)p_k}$, at each communication round $t$, we have
\begin{align}
    \mathbb{E}[\|v_k^{t+1}-v_k^*\|^2] &\leq \left(1-\frac{2}{t+1}\right) \mathbb{E} [\|v^t - v^*\|^2] + \frac{4 (G_1+\lambda (M+\frac{G_1}{\mu}))^2}{(t+1)^2 (\mu+\lambda)^2 p_k^2} + \frac{4\lambda^2}{(t+1)^2 (\mu+\lambda)^2 p_k^2} \mathbb{E}[\|w^t-w^*\|^2] \nonumber \\
    &\quad + \frac{8\lambda (G_1+\lambda (M+\frac{G_1}{\mu}))}{(t+1)^2 (\mu+\lambda)^2 p_k^2} \sqrt{\mathbb{E} [\|w^t-w^*\|^2]} +  \frac{4\lambda}{(t+1)(\mu+\lambda) p_k} \sqrt{\mathbb{E}[\|v_k^t-v_k^*\|^2] \mathbb{E}[\|w^t-w^*\|^2]}. \label{eq:lam_M_1}
\end{align}
\end{lemma}

\begin{proof}
Denote $g(v_k^t; w^t)$ as the stochastic gradient of $h_k(v_k^t; w^t)$. Let $I_t$ indicate if device $k$ is selected at the $t$-th round, and $\mathbb{E}[I_t] = p_k$.
\begin{align}
   \mathbb{E}[\|v_k^{t+1}-v_k^*\|^2] &= \mathbb{E}[\|v_k^t-\eta I_t g(v_k^t; w^t)-v_k^*\|^2] \\
    &= \mathbb{E}[\|v_k^t-v_k^*\|^2] + \eta^2 \mathbb{E} [\|I_t g(v_k^t; w^t)\|^2] + 2\eta \mathbb{E}\langle I_t g(v_k^t; w^t), v_k^*-v_k^t \rangle \\
    & \leq (1-(\mu+\lambda) \eta p_k) 
    \mathbb{E} [\|v_k^t-v_k^*\|^2] + \eta^2 \mathbb{E} [\|g(v_k^t; w^t)\|^2] + 2\eta p_k \mathbb{E}[h(v_k^*; w^t)-h(v_k^t; w^t)] \\
    &\leq (1-(\mu+\lambda) \eta p_k) \mathbb{E}[\|v_k^t - v_k^*\|^2] \nonumber\\
    &\quad + \eta^2 \mathbb{E}[\| g(v_k^t; w^*)\|^2] + \eta^2 \lambda^2 \mathbb{E} [\| w^t - w^*\|^2] + 2\eta^2 \lambda  \mathbb{E}[ \|g(v_k^t; w^*)\|\| w^t - w^*\|] \nonumber \\
    &\quad + 2\eta p_k (h(v_k^*; w^*) - \mathbb{E} [h(v_k^t; w^*)]) + 2 \eta p_k \lambda \mathbb{E} [\| v_k^t - v_k^* \|\| w^t - w^* \|].\label{eq:tmp}
\end{align}
Further, note that
\begin{align}
\mathbb{E}[\|v_k^t-u_k^*\|^2] &\leq \frac{1}{\mu^2} \mathbb{E}[\|\nabla F_k(v_k^t)\|^2] \leq \frac{G_1^2}{\mu^2}, \\
    \mathbb{E}[\|v_k^t-w^*\|^2] &= \mathbb{E}[\|v_k^t-u_k^*+u_k^*-w^*\|^2] \\ &\leq \mathbb{E}[\|v_k^t-u_k^*\|^2] + \mathbb{E}[\|u_k^*-w^*\|^2] + 2\mathbb{E}[\|v_k^t-u_k^*\|\|u_k^*-w^*\|] \\ &\leq \frac{G_1^2}{\mu^2}+M^2 + \frac{2MG_1}{\mu}, \\
    \mathbb{E}[\|g(v_k^t; w^*)\|^2] &= \mathbb{E}[\|\nabla F_k(v_k^t) + \lambda (v_k^t-w^*)\|^2] \\ &\leq G_1^2 + \lambda^2 (\frac{G_1}{\mu}+M)^2 + 2G_1 \lambda (\frac{G_1}{\mu}+M).
\end{align}
Plug it into~\eqref{eq:tmp}, 
\begin{align}
    \mathbb{E}[\|v_k^{t+1}-v_k^*\|^2] & \leq (1-(\mu+\lambda) \eta p_k) \mathbb{E}[\|v_k^t - v_k^*\|^2] + \eta^2 (G_1+\lambda(M+\frac{G_1}{\mu}))^2 + \eta^2 \lambda^2 \mathbb{E} [\|w^t-w^*\|^2] \nonumber \\
    & \quad + 2\eta^2 \lambda (G_1+\lambda(M+\frac{G_1}{\mu})) \sqrt{\mathbb{E}[\|w^t-w^*\|^2]} + 2\eta p_k \lambda  \sqrt{\mathbb{E} [\|v_k^t-v_k^*\|^2] \mathbb{E} [\|w^t-w^*\|^2]}.
\end{align}
where the last step is due to $E[XY] \leq \sqrt{E[X^2]E[Y^2]}$.
The Lemma then holds by taking $\eta=\frac{2}{(t+1)(\mu+\lambda)p_k}$.
\end{proof}

Lemma~\ref{lemma:v_w_stochastic} relates  $\mathbb{E}[\|v_k^{t+1}-v_k^*\|^2]$ with $\mathbb{E}[\|v_k^{t}-v_k^*\|^2]$ and $\mathbb{E}[\|w_k^{t}-w^*\|^2]$. Based on this, we prove that personalized models can inherit the convergence rate of the global model $w^t$ for any $G(\cdot)$.
\begin{theorem}[Relations between convergence of  global and personalized models]\label{thm:w_t_v_t}
Under the assumptions above, if \textcolor{black}{there exists a constant $A$ such that $\frac{g(t+1)}{g(t)} \geq 1-\frac{g(t)}{A}$}, then there exists $C<\infty$ such that for any device $k \in [K]$, $\mathbb{E}[\|v_k^t-v_k^*\|^2] \leq C g(t)$ with a local learning rate $\eta = \frac{2g(t)}{A(\mu+\lambda)p_k}$.
\end{theorem}
\begin{proof}
We proceed the proof by induction. First, for any constant $C > \frac{\mathbb{E}[\|v_k^0-v_k^*\|^2]}{g(0)}$,  $\mathbb{E}[\|v_k^0-v_k^*\|^2] \leq C g(0)$. If $\mathbb{E}[\|v_k^t-v_k^*\|^2] \leq C g(t)$ holds, then for $t+1$, from Lemma~\ref{lemma:v_w_stochastic}, 
\begin{align}
    \mathbb{E}[\|v_{k+1}^t-v_k^*\|^2] &\leq \left(1-\frac{2g(t)}{A}\right)  C g(t) \nonumber \\
    &\quad + \frac{g(t)^2}{A} \frac{4}{A p_k^2} \left(\frac{(G_1+\lambda (M+\frac{G_1}{\mu}))^2}{(\mu+\lambda)^2} + g(t) + \frac{2 (G_1+\lambda (M+\frac{G_1}{\mu})) \sqrt{g(t)}}{\mu+\lambda} \right) + g(t)^2 \frac{4\lambda \sqrt{C}}{(\mu+\lambda)} \label{eq:lam_M_2} \\
    & \leq \left(1-\frac{2g(t)}{A}\right)  C g(t) + \frac{C g(t)^2}{A}
\end{align}
holds  for some $C < \infty$. Hence, 
\begin{align}
     \mathbb{E}[\|v_{k+1}^t-v_k^*\|^2] &\leq \left(1-\frac{2g(t)}{A}\right)  C g(t) + \frac{Cg(t)^2}{A} \\
     &= \left(1-\frac{g(t)}{A}\right) Cg(t) \\
     &\leq Cg(t+1),
\end{align}
completing the proof.
\end{proof}
\textcolor{black}{\paragraph{Discussions.}  Theorem~\ref{thm:w_t_v_t} also suggests how the percentage/power of malicious devices can affect convergence rates. The percentage/power of adversaries impacts both the optimal global solution $w^*$, and the convergence rate of the global model $g(t)$. (i) For $w^*$, it affects $M$ in Eq~\eqref{eq:heterogeneity_assum}---the distance between the local model on a benign device and the global model. This in turn affects $\lambda$ in Eq~\eqref{eq:lam_M_1} and~\eqref{eq:lam_M_2}, and the constant $C$. $\lambda$ can scale inversely proportional to $M$, which is consistent with our fairness/robustness analysis where $\lambda^*$ should decrease as the increase of $\tau^2$. (ii) For $g(t)$, the modularity of Ditto  allows for decoupling the convergence of personalized models and the global model (as demonstrated by this theorem), and we can plug in any previous algorithms and their analysis on the convergence rate $g(t)$ as a function of malicious devices.}

As a direct result of Theorem~\ref{thm:w_t_v_t}, we could state a result for \ditto when the global objective is FedAvg.
\begin{corollary}[Convergence of personalized models]\label{coro:convergence_stochastic}
Under the assumptions above, if the global objective $G(\cdot)$ is FedAvg, then under Algorithm~\ref{alg:1_fedavg}, for $k \in [K]$,
\begin{align}
    \mathbb{E}[\|v_k^t-v_k^*\|^2] = O(1/t).
\end{align}
\end{corollary} 
\begin{proof}
From~\citet{li2019convergence} Theorem 2, we know the global model for FedAvg converges at a rate of $O(1/t)$, i.e.,
\begin{align}
    \mathbb{E} [\|w^t-w^*\|^2] \leq \frac{D'}{t+B} \|w^1-w^*\|^2 \leq \frac{D}{t+1},
\end{align}
where $D, D', B$ are constants.
Setting $g(t) = \frac{D}{t+1}$ and $A=D$ in Theorem~\ref{thm:w_t_v_t}, it follows that $\mathbb{E}[\|v_k^t-v_k^*\|^2] = O(1/t)$.
\end{proof}

\newpage
\section{Experimental Details}\label{app:exp:detail}

\subsection{Datasets and Models} \label{app:exp:data}
We summarize the datasets, corresponding models, and tasks in Table~\ref{table: data} below. We evaluate the performance of \ditto with both convex and non-convex models across a set of FL benchmarks. In our datasets, we have both image data (FEMNIST, CelebA, Fashion MNIST), and text data (StackOverflow). 
\setlength{\tabcolsep}{2pt}
\begin{table}[h!]
	\begin{center}
		\caption{\small Summary of datasets.}
		\label{table: data}
		\vspace{1em}
		\scalebox{0.84}{
		\begin{tabular}{ llllll } 
			\toprule
			\textbf{Datasets}  & \textbf{ \# Devices} & \textbf{Data Partitions} & \textbf{Models} & \textbf{Tasks} \\
			\hline
			Vehicle~\cite{duarte2004vehicle}\footnotemark  & ~~~~23 & natural (each device is a vehicle) & linear SVM  & binary classification  \\
			FEMNIST~\cite{cohen2017emnist} & ~~~~205 & natural (each device is a writer) & CNN & 62-class classification \\
			CelebA~\cite{liu2015faceattributes} & ~~~~515 & natural (each device is a celebrity) & CNN & binary classification \\
			Fashion MNIST~\cite{xiao2017fashion} & ~~~~500  & synthetic (assign 5 classes to each device) &  CNN & 10-class classification  \\
			StackOverflow~\cite{TFF}\footnotemark &  ~~~~400& natural (each device is a user) & logistic regression & 500-class tag prediction \\
			FEMNIST (skewed)~\cite{cohen2017emnist} & ~~~~100  & synthetic (assign 5 classes to each device)~ & CNN & 62-class classification \\
			\bottomrule
		\end{tabular}}
	\end{center}
\end{table}

FEMNIST is Federated EMNIST, which is EMNIST~\cite{cohen2017emnist} partitioned by the writers of digits/characters created by a previous federated learning benchmark~\cite{caldas2018leaf}. We have two versions of FEMNIST in this work under different partitions with different levels of statistical heterogeneity. The manually-partitioned version is more heterogeneous than the naturally-partitioned one, as we assign 5 classes to each device. We show that the benefits of \ditto can be more significant on the skewed FEMNIST data (Table~\ref{table:femnist_full_skewed}). All results shown in the main text are based on the natural partition. We downsample the number of data points on each device (following the power law) for Vehicle. For FEMNIST, CelebA, and StackOverflow, we randomly sample devices (users) from the entire dataset. We use the full version of Fashion MNIST (which has been used in previous FL works~\cite{bhagoji2019analyzing}), and assign 5 classes to each device.

\footnotetext[2]{\url{http://www.ecs.umass.edu/~mduarte/Software.html}}
\footnotetext[3]{\url{https://www.tensorflow.org/federated/api\_docs/python/tff/simulation/datasets/stackoverflow/load\_data.}}

\subsection{Personalization Baselines} \label{app:exp:baseline}
We elaborate on the personalization baselines used in our experiments  (Table~\ref{table: compare_other_mtl}) which allow for partial device participation and local updating. We consider:
\begin{itemize}[leftmargin=*]
    \item \textbf{MOCHA}~\cite{smith2017federated}, a primal-dual framework for multi-task learning. It jointly learns the model parameters and a device relation matrix, and applicable to convex problems.
    \item \textbf{APFL}~\cite{deng2020adaptive}, which proposes to interpolate between local and global models for personalization. While it can reduce to solving local problems (without constraints on the solution space) as pointed out in~\cite{deng2020adaptive}, we find that in neural network applications, it has some personalization benefits, possibly due to the joint optimization solver.
    \item \textbf{Elastic Weight Consolidation (EWC)}, which takes into account the Fisher information when finetuning from the optimal global model~\cite{kirkpatrick2017overcoming, yu2020salvaging}. The local objective is $\min_{w} F_k(w) + \frac{\lambda}{2} \sum_i \mathbf{F}_{ii} \cdot (w[i]-w^*[i])^2$ where $[i]$ denotes the index of parameters and $\mathbf{F}_{ii}$ denotes the $i$-th diagonal of the empirical Fisher matrix $\mathbf{F}$ estimated using a data batch.
    \item \textbf{L2SGD}, which regularizes personalized models towards their mean~\cite{hanzely2020federated}. The proposed method requires full device participation once in a while. However, to remain consistent with the other solvers, we use their objective but adopt a different solver with partial device participation---each selected local device solving $\min_w F_k(w) + \frac{\lambda}{2} \|w - \bar{w}\|^2$ where $\bar{w}$ is the current mean of all personalized models $\bar{w}=\frac{1}{N} \sum_{k=1}^N w_k$. 
    
    \item \textbf{Mapper}, which is one of the three personalization methods proposed in~\citet{mansour2020three} that needs the minimal amount of meta-information. Similar to APFL, it is also motivated by model interpolation.
    
     \item  \textbf{Per-FedAvg} (HF)~\cite{fallah2020personalized} which applies MAML~\cite{finn2017model} to personalize federated models with an Hessian-product approximation to approximate the second-order gradients. 
 
    \item \textbf{Symmetrized KL} 
    constrains the symmetrized KL divergence between the prediction of finetuned models and that of the initialization. Specifically, in our setting, the local objective is 
    $\min_w F_k(w) + \frac{\lambda}{2} \left(D_{\text{KL}}(f(w) || f(w^*)) + D_{\text{KL}}(f(w^*) || f(w))\right)$ where $D_{\text{KL}}(P || Q)$ is the KL-divergence between $P$ and $Q$, and $f(\cdot)$ denotes the softmax probability for classification.

\end{itemize}

\section{Additional and Complete Experiment Results} \label{app:exp:full}

\subsection{Comparing with Finetuning} \label{app:exp:full:two_solver}
As discussed in Section~\ref{sec:solver}, finetuning on $h_k$ for each device $k$ is a possible solver for \ditto. In non-convex cases, however, starting from a corrupted $w^*$ may result in inferior performance compared with Algorithm~\ref{alg:1}. We provide a simple example to illustrate this point. To perform finetuning, we run different numbers of epochs of mini-batch SGD on the \ditto objective for each device in the network, and pick the best one. 
As shown in Figure~\ref{fig:finetuning_full} below, finetuning at round 5,000 will not result in a good final accuracy.  We observe that one could also stop at early iterations and then finetune. However, it is difficult to do so in practice based on the training or validation data alone, as shown in Figure~\ref{fig:loss_curve}.

\begin{figure}[h]
\centering
\begin{minipage}{.44\textwidth}
  \centering
  \includegraphics[width=0.75\linewidth]{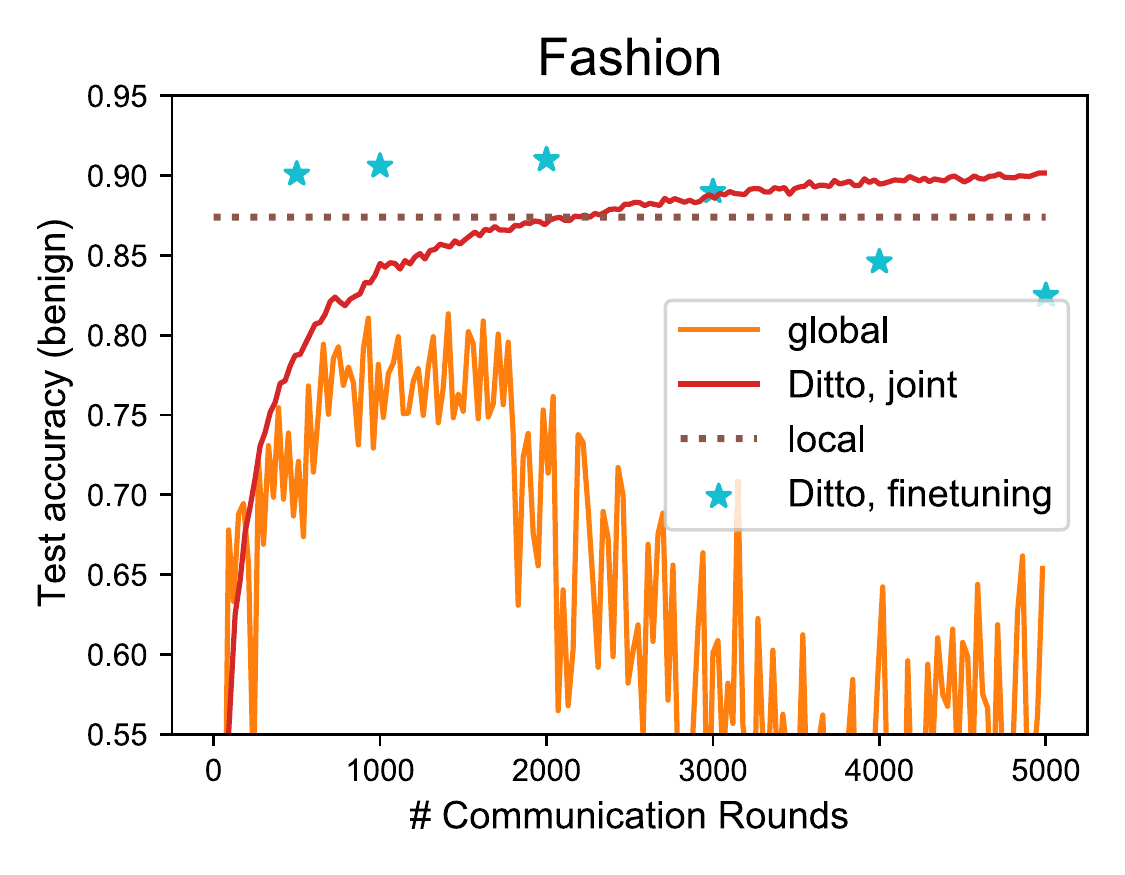}
  \vspace{-0.15in}
  \captionof{figure}{`\ditto, joint' achieves high test accuracy on benign devices. The performance can also be good if we first early stop at some specific points and then finetune.}
  \label{fig:finetuning_full}
\end{minipage}
\hfill
\begin{minipage}{.53\textwidth}
        \centering
        \begin{subfigure}{0.48\textwidth}
        \includegraphics[width=0.99\textwidth]{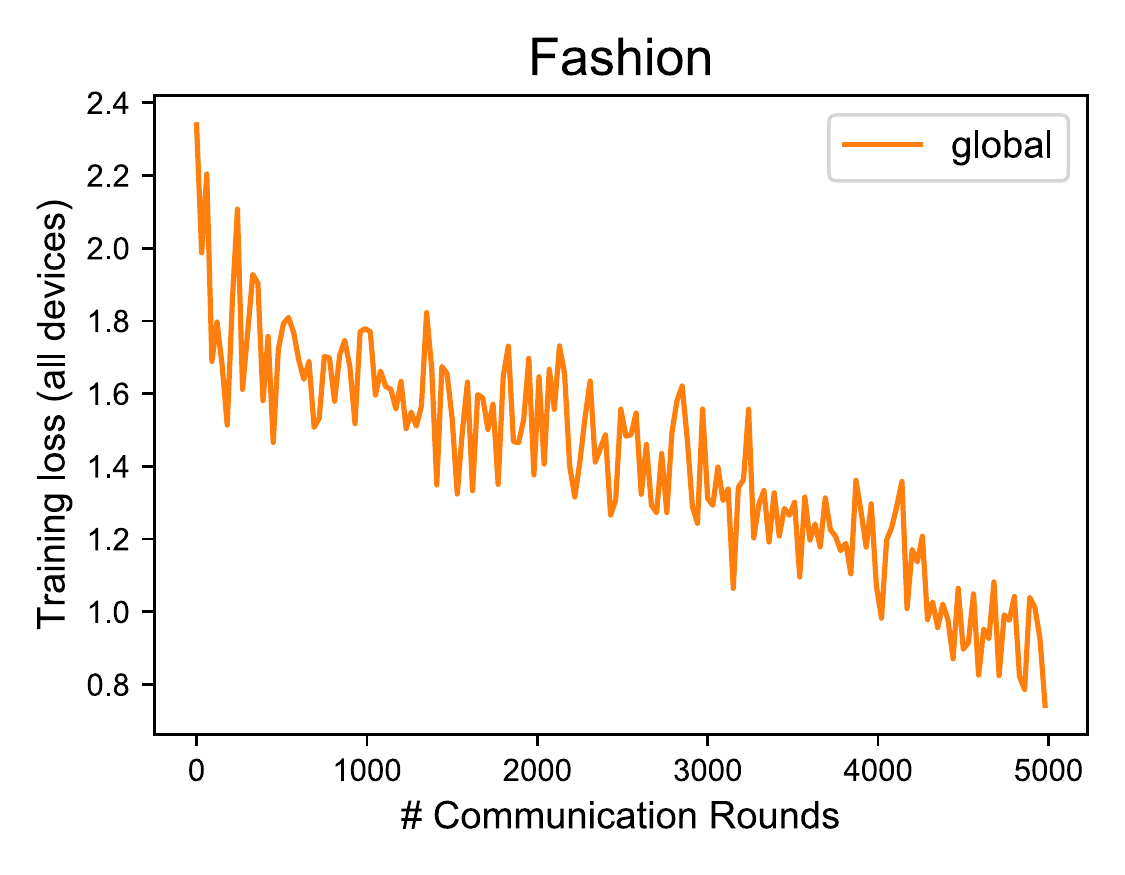}
        \end{subfigure}
        \hfill
        \begin{subfigure}{0.48\textwidth}
        \includegraphics[width=0.99\textwidth]{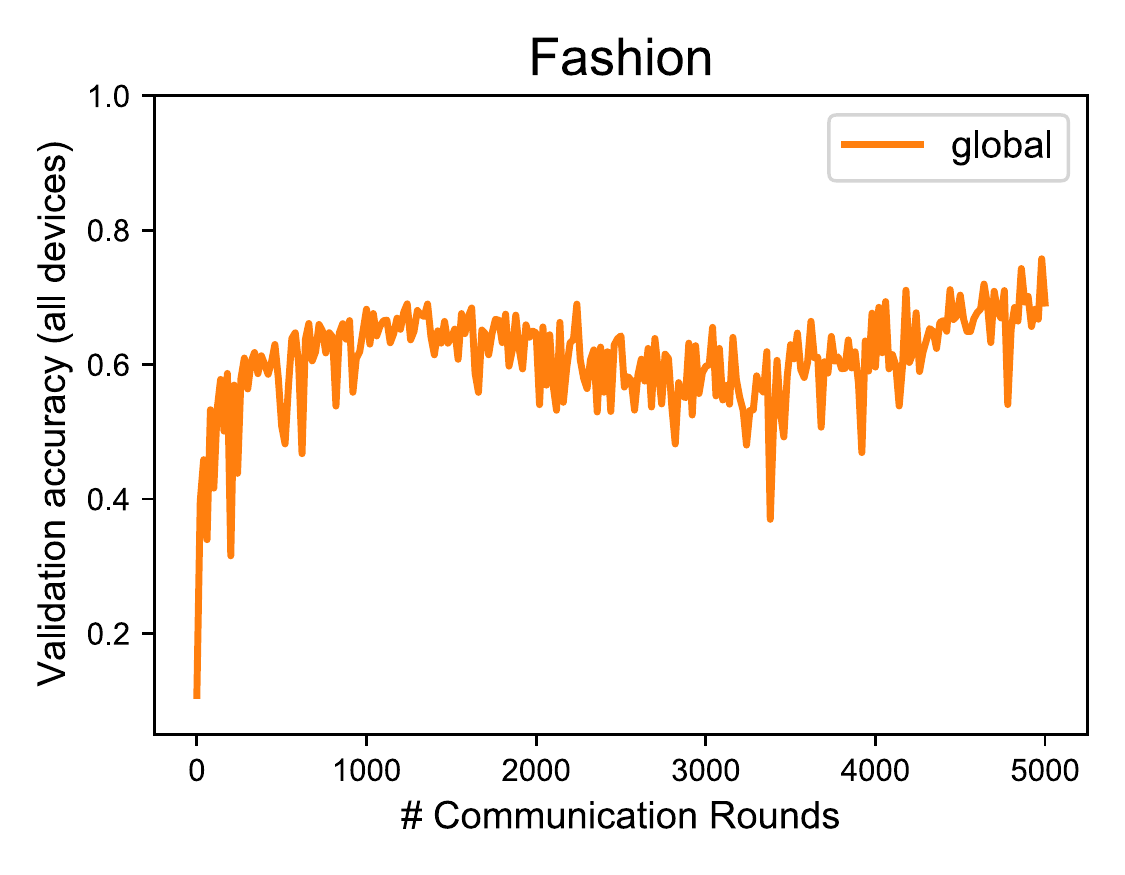}
        \end{subfigure}
        \captionof{figure}{Finetuning is not very practical as it is difficult to determine when to stop training the global model by looking at the training loss (left) or validation accuracy (right) on all devices (without knowing which are benign).}  
        \label{fig:loss_curve}
\end{minipage}
\end{figure}

\subsection{Tuning $\lambda$} \label{app:exp:full:lambda}

We assume that the server \textit{does not} have knowledge of which devices are benign vs. malicious, and we have each device \textit{locally} select and apply a best $\lambda$ from a candidate set of three values based on their validation data. For benign devices, this means they will pick a $\lambda$ based on their clean validation signal. For malicious devices, how they perform personalization (i.e., selecting $\lambda$) does not affect the corrupted global model updates they send, which are independent of $\lambda$. 
We further assume the devices have some knowledge of how `strong' the attack is.
We define strong attacks as (i) all of model replacement attacks (A3) where the magnitude of the model updates from malicious devices can scale by $>10\times$, and (ii) other attacks where more than half of the devices are corrupted.
In particular, for devices with very few validation samples (less than 4), we use a fixed small $\lambda$ ($\lambda$=0.1) for strong attacks, and use a fixed relatively large $\lambda$ ($\lambda$=1) for all other attacks. For devices with more than 5 validation data points, we let each select $\lambda$ from $\{0.05, 0.1, 0.2\}$ for strong attacks, and select $\lambda$ from $\{0.1, 1, 2\}$ for all other attacks. For the StackOverflow dataset, we tune $\lambda$ from $\{0.01,0.05,0.1\}$ for strong attacks, and $\{0.05,0.1,0.3\}$ for all other attacks. We directly evaluate our hyperparameter tuning strategy in Table~\ref{table:tune_lambda} below---showing that this dynamic tuning heuristic works well relative to an ideal, but more unrealistic strategy that picks the best $\lambda$ based on knowledge of which devices are benign vs. malicious (i.e., by only using the validation data of the benign devices).

\setlength{\tabcolsep}{2pt}
\begin{table}[h]
	\caption{Results (test accuracy and standard deviation) of using dynamic $\lambda$'s. `Best $\lambda$' refers to the results of selecting the best (fixed) $\lambda$ based on average validation performance on benign devices (assuming the server knows which devices are malicious).}
	\vspace{1em}
	\centering
	\label{table:tune_lambda}
	\scalebox{0.87}{
	\begin{tabular}{l c|ccc|ccc|ccc} 
	   \toprule[\heavyrulewidth]
        \textbf{FEMNIST} & & \multicolumn{3}{c|}{{\bf A1} (ratio of adversaries)}  &  \multicolumn{3}{c|}{{\bf A2} (ratio of adversaries)} & \multicolumn{3}{c}{{\bf A3} (ratio of adversaries)}\\
        \cmidrule(r){3-11}
         Methods &  clean & 20\% & 50\% & 80\% & 20\% & 50\% & 80\% & 10\% & 15\% & 20\%  \\
        \midrule
        best $\lambda$	    &  0.836 (.10) & 0.803 (.10) & 0.767 (.10) & 0.672 (.14) & 0.792 (.11) &  0.743 (.14) & 0.674 (.14) & 0.691 (.15) & 0.664 (.14) & 0.650 (.14) \\
        dynamic $\lambda$'s & 0.834 (.09) & 0.802 (.10) & 0.762 (.11) & 0.672 (.13) & 0.801 (.09) & 0.700 (.15) & 0.675 (.14) & 0.685 (.15) & 0.650 (.14) & 0.613 (.13)\\
        \hline
        \hline
        \textbf{Fashion} & & \multicolumn{3}{c|}{{\bf A1} (ratio of adversaries)}  &  \multicolumn{3}{c|}{{\bf A2} (ratio of adversaries)} & \multicolumn{3}{c}{{\bf A3} (ratio of adversaries)}\\
        \cmidrule(r){3-11}
         Methods &  clean & 20\% & 50\% & 80\% & 20\% & 50\% & 80\% & 10\% & 20\% & 50\%  \\
        \midrule
        best $\lambda$  & {0.946 (.06)}	& {0.944 (.08)}  & {0.935 (.07)} &	{0.925 (.07)}	& {0.943 (.08)}	& {0.930 (.07)} & {0.912 (.08)} & {0.914 (.09)} & {0.903 (.09)} & {0.873 (.09)} \\
        dynamic $\lambda$'s & 0.943 (.06) & 0.944 (.07) & 0.937 (.07) & 0.907 (.10) & 0.938 (.07) & 0.930 (.08) & 0.913 (.09) & 0.921 (.09) & 0.902 (.09) & 0.872 (.11) \\
        \hline
        \hline
        \textbf{CelebA} & & \multicolumn{3}{c|}{{\bf A1} (ratio of adversaries)}  &  \multicolumn{3}{c|}{{\bf A2} (ratio of adversaries)} & \multicolumn{3}{c}{{\bf A3} (ratio of adversaries)}\\
        \cmidrule(r){3-11}
         Methods &  clean & 20\% & 50\% & 80\% & 20\% & 50\% & 80\% & 10\% & 15\% & 20\%  \\
        \midrule
        best $\lambda$	&  0.914 (.18) & 0.828 (.22) & 0.721 (.27) & 0.724 (.28) & 0.872 (.22) &  0.826 (.26) & 0 708 (.29) & 0.699 (.28) & 0.694 (.27) & 0.689 (.28) \\
        dynamic $\lambda$'s & 0.911 (.16) & 0.820 (.26) & 0.714 (.28) & 0.724 (.28) & 0.872 (.22) & 0.826 (.26) & 0.706 (.28) & 0.699 (.28) & 0.694 (.27) & 0.689 (.28) \\
        \hline
        \hline
        \textbf{Vehicle} & & \multicolumn{3}{c|}{{\bf A1} (ratio of adversaries)}  &  \multicolumn{3}{c|}{{\bf A2} (ratio of adversaries)} & \multicolumn{3}{c}{{\bf A3} (ratio of adversaries)}\\
        \cmidrule(r){3-11}
         Methods &  clean & 20\% & 50\% & 80\% & 20\% & 50\% & 80\% & 10\% & 20\% & 50\%  \\
        \midrule
        best $\lambda$	& 0.882 (.05) & 0.862 (.05) & 0.841 (.09) & 0.851 (.06) & 0.884 (.05) & 0.872 (.06) & 0.879 (.04) & 0.872 (.06) & 0.829 (.08) & 0.827 (.08) \\
        dynamic $\lambda$'s & 0.872 (.05) & 0.857 (.06) & 0.827 (.08) & 0.834 (.05) & 0.872 (.06) & 0.867 (.07) & 0.848 (.04)  & 0.839 (.08) & 0.824 (.08) & 0.822 (.09) \\
        \hline
        \hline
        \textbf{StackOverflow} & & \multicolumn{3}{c|}{{\bf A1} (ratio of adversaries)}  &  \multicolumn{3}{c|}{{\bf A2} (ratio of adversaries)} & \multicolumn{3}{c}{{\bf A3} (ratio of adversaries)}\\
        \cmidrule(r){3-11}
         Methods &  clean & 20\% & 50\% & 80\% & 20\% & 50\% & 80\% & 10\% & 20\% & 50\%  \\
        \midrule
        best $\lambda$	& 0.315 (.16)& 0.325 (.16)&  0.315 (.17)&  0.313 (.15)&  0.314 (.16)&  0.350 (.16)&  0.312 (.14)&  0.316 (.17)&  0.321 (.17)& 0.327 (.17) \\
        dynamic $\lambda$'s & 0.317 (.17) & 0.323 (.18) & 0.314 (.16) &0.359 (.16) & 0.326 (.17) & 0.317 (.17) & 0.301 (.17)  & 0.318 (.17) & 0.319 (.17) & 0.311 (.17) \\
    \bottomrule[\heavyrulewidth]
	\end{tabular}}
\end{table}

\newpage
\subsection{\ditto Augmented with Robust Baselines} \label{app:exp:full:robustify}
In Section~\ref{sec:exp:other_properties}, we demonstrate that the performance of \ditto can be further improved when it is combined with robust baselines (e.g., learning a robust $w^*$ via robust aggregation). Here, we report full results validating this claim in Table~\ref{table:ditto+robust_baseline_full} below.

\setlength{\tabcolsep}{2pt}
\begin{table}[h!]
	\caption{\ditto augmented with robust baselines (full results).}
	\centering
	\vspace{1em}
	\label{table:ditto+robust_baseline_full}
	\scalebox{0.9}{
	\begin{tabular}{l ccc|ccc|ccc} 
	   \toprule[\heavyrulewidth]
        \textbf{FEMNIST} & \multicolumn{3}{c|}{{\bf A1} (ratio of adversaries)}  & \multicolumn{3}{c|}{{\bf A2} (ratio of adversaries)} & \multicolumn{3}{c}{{\bf A3} (ratio of adversaries)} \\
        \cmidrule(r){2-10}
         Methods &  20\%  & 50\% & 80\%  & 20\%  & 50\% & 80\%  &  10\%  & 15\% & 20\%   \\
        \midrule
        global & 0.773 (.11) & 0.727 (.12) & 0.574 (.15) & 0.774 (.11) & 0.703 (.14) & 0.636 (.15) & 0.517 (.14)	  & 0.487 (.14)  & 0.364 (.13)  \\
        clipping &  0.791 (.11)	& 0.736 (.11) & 0.408 (.14)	& 0.791 (.11)	& 0.736 (.13) & 0.656 (.13)	  & 0.795 (.11)	& 0.060 (.05)  & 0.061 (.05) \\
        \ditto	    & 0.803 (.10)	& \textbf{0.767 (.10)} & \textbf{0.672} (.14) & 0.792 (.11)  & 0.743 (.14) & 0.674 (.14)	    & 0.691 (.15)	&	0.664 (.14) &  0.650 (.14) \\
        \ditto + clipping & \textbf{0.810 (.11)} & 0.762 (.11) & 0.645 (.13) & \textbf{0.808 (.11)}	& \textbf{0.757 (.11)} & \textbf{0.684 (.13)}	& \textbf{0.813 (.13)} & \textbf{0.707 (.15)}	&  \textbf{0.672 (.14)} \\
        \hline
        \hline
        \textbf{CelebA} & \multicolumn{3}{c|}{{\bf A1} (ratio of adversaries)}  & \multicolumn{3}{c|}{{\bf A2} (ratio of adversaries)} & \multicolumn{3}{c}{{\bf A3} (ratio of adversaries)} \\
        \cmidrule(r){2-10}
         Methods &  20\% & 50\% & 80\%  & 20\% & 50\% & 80\%  &  10\%  & 15\% & 20\%  \\
        \midrule
        global & 0.810 (.22)	& 0.535 (.26) & 0.228 (.21)	& 0.869 (.22)	& 0.823 (.23) & 0.656 (.26) & 0.451 (.27) & 0.460 (.29) & 0.515 (.31) \\
        multi-Krum	        & \textbf{0.882 (.22)}	& 0.564 (.26) & 0.107 (.19)	& 0.887 (.21) & 0.891 (.20)	& 0.617 (.30) & 0.512 (.27) & 0.529 (.27) & 0.430 (.26) \\ 
        \ditto    & 0.828 (.22)	& 0.721 (.27) & 0.724 (.28) & 0.872 (.22)  & 0.826 (.26)  &  0.708 (.29) & 0.699 (.28)	& 0.694 (.27) & 0.689 (.28)   \\
        \ditto + multi-Krum & 0.875 (.20) & \textbf{0.722 (.26)} & \textbf{0.733 (.27)} & \textbf{0.903 (.20)} & \textbf{0.902 (.21)} & \textbf{0.885 (.23)} & \textbf{0.713 (.28)} & \textbf{0.709 (.28)} &  \textbf{0.713 (.28)} \\
    \bottomrule[\heavyrulewidth]
	\end{tabular}}
\end{table}

\subsection{\ditto Complete Results} \label{app:exp:full:big_table}

In Section~\ref{sec:exp:robust}, we present partial results on three strong attacks on two datasets. Here, we provide full results showing the robustness and fairness of \ditto on all attacks and all datasets compared with all defense baselines. We randomly split local data on each device into 72\% train, 8\% validation, and 20\% test sets, and report all results on test data. {We use a learning rate of 0.01 for StackOverflow, 0.05 for Fashion MNIST and 0.1 for all other datasets; and batch size 16 for CelebA and Fashion MNIST, 32 for FEMNIST and Vehicle, and 100 for StackOverflow.} 
For every dataset, we first run FedAvg on clean data to determine the number of communication rounds. Then we run the same number of rounds for all attacks on that dataset.

For our robust baselines, `median' means coordinate-wise median. For Krum, multi-Krum, $k$-norm, and $k$-loss, we assume the server knows the expected number of malicious devices when aggregation. In other words, for $k$-norm, we filter out the updates with the $k$ largest norms where $k$ is set to the expected number of malicious devices. Similarly, for $k$-loss, we only use the model update with the $k$+$1$-th largest training loss. For gradient clipping, we set the threshold to be the median of the gradient norms coming from all selected devices at each round. FedMGDA+ has an additional $\varepsilon$ hyperparameter which we select from $\{0, 0.1, 0.5, 1\}$ based on the validation performance on benign devices.  For the finetuning (only on neural network models) baseline, we run 50 epochs of mini-batch SGD on each device on the local objective $F_k$ starting from $w^*$. We see that \ditto can achieve better fairness and robustness in most cases. In particular, on average of all datasets and all attack scenarios, \ditto (with dynamic $\lambda$'s) achieves ~6\% absolute accuracy improvement compared with the strongest robust baseline. In terms of fairness, \ditto is able to reduce the variance of test accuracy by $10\%$ while improving the average accuracy by 5\% relative to state-of-the-art methods for fair FL (without attacks).

\setlength{\tabcolsep}{2pt}
\begin{table}[h]
	\caption{Full results (average and standard deviation of test accuracy across all devices) on the Vehicle dataset with linear SVM. On this convex problem, we additionally compare with another primal-dual MTL method MOCHA~\citep{smith2017federated}, which suggests the fairness/robustness benefits of other MTL approaches.}
	\vspace{1em}
	\centering
	\label{table:vehicle_full}
	\scalebox{0.87}{
	\begin{tabular}{l c|ccc|ccc|ccc} 
	   \toprule[\heavyrulewidth]
        \textbf{Vehicle} & & \multicolumn{3}{c|}{{\bf A1} (ratio of adversaries)}  &  \multicolumn{3}{c|}{{\bf A2} (ratio of adversaries)} & \multicolumn{3}{c}{{\bf A3} (ratio of adversaries)}\\
        \cmidrule(r){3-11}
         Methods &  clean & 20\% & 50\% & 80\%  & 20\%  & 50\% & 80\% & 10\% & 20\%  & 50\%  \\
        \midrule
        global	   & 0.866 (.16) &0.847 (.08) & 0.643 (.10) & 0.260 (.27)  & 0.866 (.18) & 0.840 (.21) & 0.762 (.27) & 0.854 (.17) &	0.606 (.08)  & 0.350 (.19) \\
        local 	   & 0.836 (.07) &0.835 (.08) &0.840 (.09)&\textbf{0.857 (.09)}  & 0.835 (.08)	& 0.840 (.09) & 0.857 (.09) & 0.840 (.07) &	0.835 (.08)	 & \textbf{0.840 (.09)}   \\
        fair & 0.870 (.08) & 0.721 (.06) & 0.572 (.08) & 0.404 (.13) & 0.746 (.12) & 0.704 (.15) & 0.706 (.20) & 0.775 (.13) & 0.628 (.25) & 0.448 (.11) \\
        median     & 0.863 (.16) &0.861 (.18) &0.676 (.11)&0.229 (.31)  & 0.864 (.18) & 0.838 (.21)& 0.774 (.28) & 0.867 (.17) & 0.797 (.07)	 & 0.319 (.17)   \\
        Krum       & 0.852 (.17) &0.853 (.19) &0.830 (.22) &0.221 (.32)  & 0.851 (.19) & 0.828 (.22)  & 0.780 (.31) & 0.867 (.17) & \textbf{0.866 (.18)}	 & 0.588 (.14) \\
        multi-Krum & 0.866 (.16) &0.867 (.18) &0.839 (.20) & 0.220 (.32)  & 0.867 (.18)	&0.839 (.22)& 0.770 (.31) &0.868 (.17)	&0.836 (.08)	 & 0.406 (.15) \\
        clipping   & 0.864 (.16) &0.865 (.17) &0.678 (.34)&0.234 (.30)  & 0.865 (.18) &0.839 (.22)	& 0.764 (.27) & 0.868 (.17) &	0.789 (.07)	 & 0.315 (.17) \\
        k-norm	   & 0.866 (.16) & \textbf{0.867 (.17)} &0.838 (.21)&0.222 (.32)  & 0.867 (.18) & 0.839 (.22)	& 0.778 (.31) & 0.867 (.17) &	0.844 (.09)	 & 0.458 (.16) \\
        k-loss     & 0.850 (.05) &0.755 (.03) &0.732 (.09)&0.217 (.31)  & 0.852 (.06)&0.840 (.07)	& 0.825 (.09) &	0.866 (.17) &0.692 (.08)  & 0.328 (.16) \\
        FedMGDA+ & 0.860 (.16) & 0.835 (.09) & 0.674 (.14) & 0.270 (.26) & 0.860 (.18) & 0.843 (.22) & 0.794 (.26) & 0.836 (.17) & 0.757 (.07) & 0.676 (.17) \\ 
        MOCHA & 0.880 (.04) & 0.848 (.07) & 0.832 (.08) & 0.829 (.10) & 0.846 (.06) & 0.843 (.07) & 0.833 (.10) & 0.862 (.06) & 0.844 (.07) & 0.834 (.07) \\
        \midrule
        \ditto, $\lambda$=0.1	    & 0.845 (.07)	& 0.841 (.08)& \textbf{0.841 (.09)}	& 0.851 (.06)	& 0.844 (.07) &0.848 (.08)	& 0.866 (.05)	& 0.838 (.07) & 0.829 (.08) & 0.827 (.08) \\
        \ditto, $\lambda$=1	    & 0.875 (.05)	& 0.859 (.06)&0.821 (.07)	& 0.776 (.08)	& 0.875 (.06)	&0.870 (.07)& \textbf{0.879 (.04)}	&0.860 (.07) & 0.813 (.07) & 0.757 (.08) \\
        \ditto, $\lambda$=2	    & \textbf{0.882 (.05)}	& {0.862 (.05)} & 0.800 (.07) & 0.709 (.12)	& \textbf{0.884 (.05)}	& \textbf{0.872 (.06)} & 0.869 (.04)	& \textbf{0.872 (.06)} & 0.791 (.06) & 0.690 (.09)\\
    \bottomrule[\heavyrulewidth]
	\end{tabular}}
\end{table}

\setlength{\tabcolsep}{2pt}
\begin{table}[h]
	\caption{Full results  (average and standard deviation of test accuracy across all devices) on FEMNIST.}
	\vspace{1em}
	\centering
	\label{table:femnist_full}
	\scalebox{0.87}{
	\begin{tabular}{l c|ccc|ccc|ccc} 
	   \toprule[\heavyrulewidth]
        \textbf{FEMNIST} & & \multicolumn{3}{c|}{{\bf A1} (ratio of adversaries)}  &  \multicolumn{3}{c|}{{\bf A2} (ratio of adversaries)} & \multicolumn{3}{c}{{\bf A3} (ratio of adversaries)}\\
        \cmidrule(r){3-11}
         Methods &  clean & 20\% & 50\% & 80\% & 20\% & 50\% & 80\% & 10\% & 15\% & 20\%  \\
        \midrule
        global	    & 0.804 (.11) &	0.773 (.11)	& 0.727 (.12) & 0.574 (.15)	& 0.774 (.11) & 0.703 (.14) & 0.636 (.15)	  & 0.517 (.14)	  & 0.487 (.14)  & 0.364 (.13) \\
        local       & 0.628 (.15) &	0.620 (.14)	& 0.627 (.14) & 0.607 (.13)	& 0.620 (.14) & 0.627 (.14) & 0.607 (.13)	  & 0.622 (.14)	  & 0.621 (.14)  & 0.620 (.14) \\
        fair        & 0.809 (.11) &	0.636 (.15)	& 0.562 (.13) & 0.478 (.12)	& 0.440 (.15) & 0.336 (.12) & 0.363 (.12)	  & 0.353 (.12)	  & 0.316 (.12)   & 0.299 (.11) \\
        median      & 0.733 (.14) &	0.627 (.15)	& 0.576 (.15) & 0.060 (.04)	& 0.673 (.14) & 0.645 (.14) & 0.564 (.15)	  & 0.628 (.14)	  & 0.573 (.15)  & 0.577 (.16) \\
        Krum        & 0.717 (.16) &	0.059 (.05)	& 0.096 (.07) & 0.091 (.07)	& 0.604 (.14) & 0.062 (.25) & 0.024 (.02)     & 0.699 (.15)	  & 0.719 (.13)  & 0.648 (.14) \\
        multi-Krum  & 0.804 (.11) &	0.790 (.11)	& 0.759 (.11) & 0.115 (.07)	& 0.789 (.11) & 0.762 (.11) & 0.014 (.02)     & 0.529 (.14)	  &  0.664 (.15) & 0.561 (.14) \\
        clipping    & 0.805 (.11) &	0.791 (.11)	& 0.736 (.11) & 0.408 (.14)	& 0.791 (.11)	& 0.736 (.13) & 0.656 (.13)	  & 0.795 (.11)	& 0.060 (.05)  & 0.061 (.05) \\
        k-norm      & 0.806 (.11) &	0.785 (.11)	& 0.760 (.12) & 0.060 (.05)	& 0.788 (.10) & \textbf{0.765 (.11)} & 0.011 (.02)   & 0.060 (.04) & 0.647 (.15)  & 0.562 (.15) \\
        k-loss      & 0.762 (.11) &	0.606 (.13)	& 0.599 (.13) & 0.596 (.13)	& 0.432 (.12) & 0.508 (.13) & 0.572 (.14)	& 0.060 (.04) &  0.009 (.02) & 0.006 (.01) \\
        FedMGDA+    & 0.803 (.12) & 0.794 (.12) & 0.730 (.12) & 0.057 (.04) & \textbf{0.793 (.12)} & 0.753 (.12) & 0.671 (.14)   & \textbf{0.798 (.11)}&  \textbf{0.794 (.12)} & \textbf{0.791 (.11)} \\
        finetuning & 0.815 (.09) & 0.778 (.11) & 0.734 (.12) & \textbf{0.671 (.13)} & 0.764 (.11) & 0.695 (.18) & 0.646 (.14) & 0.688 (.13) & 0.671 (.14) & 0.655 (.13) \\
        \midrule
        \ditto, $\lambda$=0.01 & 0.800 (.15) & 0.709 (.15) & 0.683 (.17) & 0.642 (.13)	& 0.701 (.14) & 0.684 (.14) & 0.645 (.14) & 0.650 (.14) & 0.628 (.14) & 0.650 (.14) \\
        \ditto, $\lambda$=0.1  & 0.827 (.10) & 0.794 (.11) & 0.755 (.13) & 0.666 (.14)	& 0.786 (.13) & 0.743 (.14) & \textbf{0.674 (.14)} & 0.691 (.15) & 0.664 (.14) & 0.640 (.14) \\
        \ditto, $\lambda$=1    & \textbf{0.836 (.10)} & \textbf{0.803 (.10)} & \textbf{0.767 (.10)} & \textbf{0.672 (.14)}	& \textbf{0.792 (.11)} & 0.691 (.17) & 0.575 (.17)	& 0.642 (.12) & 0.595 (.14) & 0.554 (.15) \\
    \bottomrule[\heavyrulewidth]
	\end{tabular}}
\end{table}

\setlength{\tabcolsep}{2pt}
\begin{table}[h]
	\caption{Full results (average and standard deviation of test accuracy across all devices) on Fashion MNIST.}
	\vspace{1em}
	\centering
	\label{table:fashion_full}
	\scalebox{0.87}{
	\begin{tabular}{l c|ccc|ccc|ccc} 
	   \toprule[\heavyrulewidth]
        \multicolumn{2}{l|}{\textbf{Fashion MNIST}} & \multicolumn{3}{c|}{{\bf A1} (ratio of adversaries)}  &  \multicolumn{3}{c|}{{\bf A2} (ratio of adversaries)} & \multicolumn{3}{c}{{\bf A3} (ratio of adversaries)}\\
        \cmidrule(r){3-11}
         Methods &  clean & 20\% & 50\% & 80\%  & 20\%  & 50\% & 80\%  & 10\%  & 20\% & 50\%  \\
        \midrule
        global     & 0.911 (.08)	& 0.897 (.08)	& 0.855 (.10) & 0.753 (.13)	& 0.900 (.08)	& 0.882 (.09) & 0.857 (.10)	& 0.753 (.10)	& 0.551 (.13) & 0.275 (.12) \\ 
        local      & 0.876 (.10)	& 0.874 (.10)	& 0.876 (.11) & 0.879 (.10)	& 0.874 (.10)	& 0.876 (.11) & 0.879 (.10)	& 0.877 (.10)	& 0.874 (.10) & \textbf{0.876 (.11)} \\
        fair       & 0.909 (.07)	& 0.751 (.12)	& 0.637 (.13) & 0.547 (.11)	& 0.731 (.13)	& 0.637 (.14) & 0.635 (.14)	& 0.653 (.13)	& 0.601 (.12) & 0.131 (.16) \\
        median     & 0.884 (.09)	& 0.853 (.10)	& 0.818 (.12) & 0.606 (.17)	& 0.885 (.09)	& 0.883 (.09) & 0.864 (.10)	& 0.856 (.09)	& 0.829 (.11) & 0.725 (.15) \\
        Krum       & 0.838 (.13)	& 0.864 (.11)	& 0.818 (.13) & 0.768 (.15)	& 0.847 (.12)	& 0.870 (.11) & 0.805 (.13)	& 0.868 (.11)	& 0.866 (.11) & 0.640 (.18) \\
        multi-Krum & 0.911 (.08)	& 0.907 (.08)	& 0.889 (.10) & 0.793 (.12)	& 0.849 (.10)	& 0.827 (.12) & 0.095 (.12)	& 0.804 (.11)	& 0.860 (.09) & 0.823 (.13) \\
        clipping   & 0.913 (.07)	& 0.905 (.08)	& 0.875 (.10) & 0.753 (.12)	& 0.904 (.08)	& 0.886 (.09) & 0.856 (.11)	& 0.901 (.08)	& 0.844 (.11) & 0.477 (.13) \\
        k-norm     & 0.911 (.08)	& 0.908 (.08)	& 0.888 (.10) & 0.118 (.08)	& 0.906 (.08)	& 0.893 (.09) & 0.096 (.07)	& 0.765 (.14)	& 0.854 (.10) & 0.828 (.12) \\
        k-loss     & 0.898 (.08)	& 0.856 (.09)	& 0.861 (.10) & 0.851 (.31)	& 0.876 (.09)	& 0.866 (.11) & 0.870 (.10)	& 0.538 (.14)	& 0.257 (.13) & 0.092 (.13) \\
        FedMGDA+   & 0.915 (.08)    & 0.907 (.08)   & 0.874 (.10) & 0.753 (.13) & 0.911 (.08)   & 0.900 (.09) & 0.873 (.10) & \textbf{0.914 (.08)} & \textbf{0.904 (.08)} & 0.869 (.10) \\
        finetuning  & 0.945 (.06) & \textbf{0.946 (.07)} & \textbf{0.935 (.07)} & 0.922 (.08) & \textbf{0.945 (.07)} & \textbf{0.930 (.08)} & \textbf{0.923 (.08)} & \textbf{0.915 (.08)} &0.871 (.11)  & 0.764 (.15) \\
        \midrule
        \ditto, $\lambda$=0.1  & 0.929 (.09)	 & 0.920 (.09)	& 0.909 (.10) & 0.897 (.10)	& 0.921 (.09)	& 0.914 (.09) & 0.905 (.08)	& \textbf{0.914 (.09)} & \textbf{0.903 (.09)} & \textbf{0.873 (.09)}  \\
        \ditto, $\lambda$=1    & \textbf{0.946 (.06)}	& \textbf{0.944 (.08)}  & \textbf{0.935 (.07)} &	\textbf{0.925 (.07)}	& \textbf{0.943 (.08)}	& \textbf{0.930 (.07)} & {0.912 (.08)}	& 0.887 (.09) & 0.831 (.10) & 0.740 (.12) \\
        \ditto, $\lambda$=2    & 0.945 (.06)	& 0.942 (.06)	& \textbf{0.935 (.07)} & 0.917 (.07)	& 0.936 (.07)	& 0.923 (.08) & 0.906 (.08)	& 0.871 (.09) & 0.785 (.11) & 0.606 (.14) \\
    \bottomrule[\heavyrulewidth]
	\end{tabular}}
\end{table}

\setlength{\tabcolsep}{2pt}
\begin{table}[h]
	\caption{Full results (average and standard deviation of test accuracy across all devices) on FEMNIST (skewed).}
	\vspace{1em}
	\centering
	\label{table:femnist_full_skewed}
	\scalebox{0.87}{
	\begin{tabular}{l c|ccc|ccc|ccc} 
	   \toprule[\heavyrulewidth]
        \textbf{FEMNIST (skewed)} & & \multicolumn{3}{c|}{{\bf A1} (ratio of adversaries)}  &  \multicolumn{3}{c|}{{\bf A2} (ratio of adversaries)} & \multicolumn{3}{c}{{\bf A3} (ratio of adversaries)}\\
        \cmidrule(r){3-11}
         Methods &  clean & 20\%  & 50\% & 80\%  & 20\%  & 50\% & 80\%  & 10\% & 15\% & 20\%  \\
        \midrule
        global	    & 0.720 (.24)	& 0.657 (.28)	& 0.585 (.30) & 0.435 (.23)	& 0.688 (.26) & 0.631 (.24) & 0.589 (.26) & 0.023 (.11)	&  0.038 (.18) & 0.039 (.18) \\
        local 	    & 0.915 (.18)	& 0.903 (.21)	& 0.937 (.18) & 0.902 (.19)	& 0.903 (.21) & 0.937 (.18) &  0.902 (.19)	& 0.881 (.21)	&  \textbf{0.912 (.18)} & \textbf{0.903 (.21)} \\
        fair	    & 0.716 (.22)	& 0.644 (.29)	& 0.545 (.29) & 0.421 (.22)	& 0.348 (.22)	& 0.321 (.16) &  0.242 (.15)	& 0.010 (.11)	& 0.042 (.10)  & 0.037 (.17) \\
        median      & 0.079 (.12)	& 0.086 (.12)	& 0.031 (.06) & 0.044 (.08)	& 0.075 (.12)	& 0.109 (.13) &  0.323 (.25)	& 0.060 (.10)	& 0.020 (.09)  & 0.033 (.07)  \\
        Krum        & 0.457 (.37)	& 0.360 (.35)	& 0.061 (.22)  & 0.127 (.27) & 0.424 (.38)	& 0.051 (.08) &  0.147 (.22)	& 0.434 (.36)	&   0.472 (.36) & 0.484 (.35) \\
        multi-Krum	& 0.725 (.25)	& 0.699 (.29)	& 0.061 (.22) & 0.271 (.21)	& 0.712 (.29)	& 0.705 (.30) &  0.584 (.28)	& 0.633 (.30)	&  0.556 (.30) & 0.526 (.28) \\
        clipping	& 0.727 (.28)	& 0.678 (.28)	& 0.604 (.34) & 0.401 (.26)	& 0.726 (.26)	&  0.711 (.26) &  0.645 (.24)	& 0.699 (.29)	&  0.674 (.28) & 0.640 (.28) \\
        k-norm	    & 0.716 (.28)	& 0.691 (.30)	& 0.396 (.36) & 0.005 (.08)	& 0.724 (.26)	& 0.721 (.29) &  0.692 (.35)   & 0.612 (.29)	&  0.599 (.30) & 0.565 (.28) \\
        k-loss      & 0.587 (.21)	& 0.526 (.29)	& 0.419 (.36) & 0.127 (.27)	& 0.555 (.23)	& 0.550 (.26) &  0.093 (.16)   & 0.003 (.08)   &  0.009 (.07) & 0.006 (.05) \\
        finetuning  & \textbf{0.948 (.11)} & 0.942 (.13) & \textbf{0.959 (.10)} & \textbf{0.946 (.10)} & \textbf{0.949 (.16)} & 0.918 (.21) & 0.621 (.11) & 0.788 (.25) & 0.740 (.27) & 0.751 (.26)\\
        \midrule
        \ditto, $\lambda$=0.01 & 0.947 (.15) & \textbf{0.945} (.18)	& 0.955 (.20) & \textbf{0.946 (.13)}	& 0.942 (.18)	& 0.949 (.15) & \textbf{0.944 (.14)}	& {0.902} (.20)	& 0.895 (.23) & 0.888 (.20) \\
        \ditto, $\lambda$=0.1 & \textbf{0.948 (.10)} & \textbf{0.945 (.14)}	& \textbf{0.959 (.12)} & {0.936 (.09)}	& \textbf{0.945 (.13)}	& \textbf{0.948 (.10)} & 0.888 (.18)	& \textbf{0.936 (.16)}	& 0.827 (.23) & 0.812 (.24) \\
        \ditto, $\lambda$=1 & 0.902 (.15)	& 0.899 (.15) & 0.907 (.15)	& 0.861 (.14)	& 0.899 (.18) &	0.818 (.22) & 0.423 (.41)  & 0.880 (.15) &	0.730 (.28) & 0.736 (.28) \\
    \bottomrule[\heavyrulewidth]
	\end{tabular}}
\end{table}

\setlength{\tabcolsep}{2pt}
\begin{table}[h]
	\caption{Full results (average and standard deviation of test accuracy across all devices) on CelebA.
	}
	\vspace{1em}
	\centering
	\label{table:celeba_full}
	\scalebox{0.87}{
	\begin{tabular}{l c|ccc|ccc|ccc} 
	   \toprule[\heavyrulewidth]
        \textbf{CelebA} & & \multicolumn{3}{c|}{{\bf A1} (ratio of adversaries)}  &  \multicolumn{3}{c|}{{\bf A2} (ratio of adversaries)} & \multicolumn{3}{c}{{\bf A3} (ratio of adversaries)}\\
        \cmidrule(r){3-11}
        Methods &  clean & 20\%  & 50\% & 80\%   & 20\%   & 50\% & 80\%   & 10\%  & 15\% & 20\%  \\
        \midrule
        global  & 0.911 (.19)	& 0.810 (.22)	& 0.535 (.26) & 0.228 (.21)	& 0.869 (.22)	& 0.823 (.23) & 0.656 (.26) & 0.451 (.27) & 0.460 (.29) & 0.515 (.31) \\
        local   & 0.692 (.27)	& 0.690 (.27)	& 0.682 (.27) & 0.681 (.26)	& 0.690 (.27)	& 0.682 (.27)  & 0.681 (.26) & 0.692 (.27) & 0.693 (.27) & 0.690 (.27) \\
        fair    & 0.905 (.17)	& 0.724 (.27)	& 0.509 (.27) & 0.195 (.21)	& 0.790 (.26)	& 0.646 (.27) & 0.646 (.27) & 0.442 (.27) & 0.426 (.28) & 0.453 (.28) \\
        median  & 0.910 (.18)	& 0.872 (.22)	&0.494 (.28) & 0.126 (.18)	& {0.901 (.20)}	& 0.864 (.20) & 0.617 (.30) & 0.885 (.20) & \textbf{0.891 (.19)} & \textbf{0.870 (.21)} \\
        Krum    & 0.775 (.25)	& 0.810 (.25)	& 0.641 (.25) & 0.377 (.10)	& 0.790 (.25)	& 0.699 (.25) & 0.584 (.27) & 0.780 (.25) & 0.728 (.25) & 0.685 (.30) \\
    multi-Krum  & 0.911 (.19)	& \textbf{0.882 (.22)}	& 0.564 (.26) & 0.107 (.19)	& 0.887 (.21) & 0.891 (.20)	& 0.617 (.30) & 0.512 (.27) & 0.529 (.27) & 0.430 (.26) \\
       clipping & 0.909 (.18)	& 0.866 (.19)	& 0.485 (.29) & 0.126 (.20)	& 0.897 (.20)	& 0.842 (.21) & 0.665 (.26) & \textbf{0.901 (.20)} & 0.883 (.21) & 0.853 (.23) \\
        k-norm  & 0.908 (.18)	& 0.870 (.22)	& 0.537 (.28) & 0.105 (.17)	& 0.874 (.23)	& \textbf{0.909 (.18)} & 0.664 (.25) & 0.506 (.28) & 0.577 (.27) & 0.449 (.28) \\
        k-loss  & 0.873 (.19)	& 0.584 (.28)	& 0.550 (.31) & 0.169 (.21)	& 0.595 (.28)	& 0.654 (.28) & 0.683 (.26) & 0.543 (.33) & 0.458 (.33) & 0.455 (.34) \\
        FedMGDA+ & 0.909 (.19) & 0.853 (.21) & 0.508 (.28) & 0.473 (.34) & \textbf{0.907 (.19)} & 0.889 (.21) & \textbf{0.782 (.26)} & 0.865 (.23) & 0.805 (.26) & 0.847 (.21) \\
        finetuning  & 0.912 (.18) & 0.814 (.24) & 0.721 (.28) & 0.691 (.29) & 0.850 (.24) & 0.800 (.25) & 0.747 (.24) & 0.665 (.28) & 0.668 (.27) & 0.673 (.28) \\
        \midrule
        \ditto, $\lambda$=0.1 & 0.884 (.24) & 0.716 (.27) & \textbf{0.721 (.27)} & \textbf{0.724 (.28)} & 0.727 (.26) & 0.708 (.28) & 0.706 (.28) & 0.699 (.28) & 0.694 (.27) & 0.689 (.28) \\
        \ditto, $\lambda$=1   & 0.911 (.16) & 0.820 (.26) & 0.714 (.28) & 0.675 (.29) & 0.872 (.22) & 0.826 (.26) & {0.708 (.29)} & 0.629 (.29) & 0.667 (.28) & 0.685 (.28) \\
        \ditto, $\lambda$=2 & \textbf{0.914 (.18)} & 0.828 (.22) & {0.698 (.27)} & 0.654 (.28) & 0.862 (.21)   & 0.791 (.26) & 0.623 (.31)   & 0.585 (.29)   & 0.647 (.27) & 0.655 (.29) \\
    \bottomrule[\heavyrulewidth]
	\end{tabular}}
\end{table}

\setlength{\tabcolsep}{2pt}
\begin{table}[h]
	\caption{Full results (average and standard deviation of test accuracy across all devices) on StackOverflow.
	}
	\vspace{1em}
	\centering
	\label{table:so_full}
	\scalebox{0.87}{
	\begin{tabular}{l c|ccc|ccc|ccc} 
	   \toprule[\heavyrulewidth]
        \textbf{StackOverflow} & & \multicolumn{3}{c|}{{\bf A1} (ratio of adversaries)}  &  \multicolumn{3}{c|}{{\bf A2} (ratio of adversaries)} & \multicolumn{3}{c}{{\bf A3} (ratio of adversaries)}\\
        \cmidrule(r){3-11}
        Methods &  clean & 20\%  & 50\% & 80\%   & 20\%   & 50\% & 80\%   & 10\%  & 15\% & 20\%  \\
        \midrule
        global  & 0.155 (.13)	&	0.153 (.13)& 0.156 (.16) & 0.169 (.18)   & 	0.147 (.12) &  0.009 (.03)&  0.013 (.01) & 0.000 (.00) & 0.000 (.00) & 0.000 (.00) \\
        local   & 0.311 (.15)	& 0.311 (.15)	& 0.313 (.15) & \textbf{0.319 (.15})	& 	0.311 (.15)	& 0.313 (.15) & \textbf{0.319 (.15)}  & 0.311 (.15)	& 0.313 (.15) & 0.319 (.15) \\
        fair    & 0.154 (.13) &  0.155 (.14)&  0.153 (.13)&  0.141 (.10) & 0.000 (.00) & 0.000 (.00) & 0.000 (.00)  & 0.148 (.12) &  0.152 (.13)& 0.167 (.11)\\
        median  & 0.002 (.00)	& 0.001 (.00)	&0.000 (.00) & 0.000 (.00)	& 0.000 (.00)	& 0.001 (.00) &  0.000 (.00)& 0.000 (.00) & 0.000 (.00) & 0.000 (.00) \\
        Krum    & 0.154 (.13)	& 0.150 (.13)	& 0.041 (.04) & 0.002 (.00)	& 0.158 (.13)&  0.151 (.13)&  0.167 (.12)& 0.153 (.13) & 0.154 (.14) & 0.138 (.15) \\
       clipping & 0.154 (.13)	& 0.157 (.13)	& 0.149 (.13) & 0.163 (.17)	&	0.152 (.13)&  0.001 (.01)& 0.001 (.01)& 0.155 (.12) & 0.161 (.14) & 0.120 (.16) \\
        k-norm  & 0.154 (.13) &  0.156 (.12)&  0.100 (.08)&  0.002 (.00)&  0.086 (.11)& 0.042 (.03) &  0.001 (.00)&  0.149 (.15)&  0.144 (.15)& 0.155 (.13)\\
        k-loss  &  0.155 (.13)&  0.160 (.12)&  0.164 (.13)&  0.129 (.14)&  0.136 (.11)&  0.145 (.11)&  0.156 (.14)&  0.148 (.14)&  0.159 (.13)& 0.156 (.13)\\
        FedMGDA+ &0.155 (.12) & 0.154 (.13)& 0.152 (.13)& 0.165 (.13)& 0.147 (.13)& 0.160 (.14)& 0.101 (.09)& 0.155 (.13)& 0.158 (.12)&0.154 (.13)  \\
        \midrule
        \ditto, $\lambda$=0.05 &  \textbf{0.315 (.16)}&  \textbf{0.325 (.16)}&  \textbf{0.315 (.17)}&  0.313 (.15)&  \textbf{0.314 (.16)}&  \textbf{0.350 (.16)}&  0.312 (.14)&  0.316 (.17)&  \textbf{0.321 (.17)}& \textbf{0.327 (.17)}\\
        \ditto, $\lambda$=0.1 &  0.309 (.17)&  0.318 (.17)&  \textbf{0.315 (.17)}&  0.293 (.13)&  0.309 (.17)&  0.316 (.16)&  0.307 (.14)&  \textbf{0.319 (.17)}&  0.302 (.17)& 0.305 (.17)\\
        \ditto, $\lambda$=0.3 & 0.255 (.18) &  0.298 (.18)&  0.288 (.17)&  0.304 (.16)&  0.283 (.17)&  0.233 (.18)&  0.321 (.20) & 0.252 (.17)  &  0.261 (.19) & 0.269 (.17) \\
    \bottomrule[\heavyrulewidth]
	\end{tabular}}
\end{table}

\end{document}